\documentclass[a4paper]{styles/svproc}
\usepackage{graphicx}
\usepackage{amsmath,amssymb}
\usepackage{subcaption}
\usepackage{amsfonts}
\usepackage{hyperref}
\usepackage{enumitem}
\usepackage{comment}
\usepackage{algorithm}
\usepackage{algpseudocode}
\usepackage{ifmtarg}
\usepackage{listings}

\newcommand{\qedsymbol}{\ensuremath{\square}}

\usepackage[utf8]{inputenc}

\usepackage{xcolor}

\usepackage{xspace}

\def\R{\mathbb{R}}
\def\N{\mathbb{N}}
\newcommand\QRRT[1]{\ensuremath{\text{QRRT\_}{(#1)}}} 

\newcommand\subfigE[2]{
\begin{subfigure}[b]{0.31\linewidth}
   \includegraphics[width=\linewidth]{#1}
   \caption{#2}
\end{subfigure}
}
\newcommand{\bigslant}[2]{{\raisebox{.2em}{$#1$}\left/\raisebox{-.2em}{$#2$}\right.}}

\def\Z{Z}

\def\free{\textnormal{free}}
\def\obs{\textnormal{obs}}
\def\X{X}
\def\Xf{\X_{\free}}
\def\Xo{\X_{\text{\obs}}}
\def\xprime{x^{\prime}}
\def\xi{x_I}
\def\xg{x_G}

\def\tX{Y}
\def\tx{y}
\def\tXf{\tX_{\free}}
\def\tXo{\tX_{\obs}}

\def\tildeX{\tilde{\X}}
\def\tildeXf{\tildeX_{\free}}

\def\PriorityQueue{\ensuremath{\mathbf{Q}}}

\def\planningspace{\left(\X,\phi\right)}
\def\planningproblem{\left(\Xf,\xi, \xg\right)}
\def\alds{\big(\pi,\tphi\big)}

\def\tphi{\phi_{\tX}}

\def\hstar{h^{*}}

\def\xr{x_{\text{rand}}}
\def\xn{x_{\text{near}}}
\def\xw{x_{\text{new}}}
\def\G{\ensuremath{\mathbf{T}}}
\def\path{\ensuremath{\mathbf{p}}}

\makeatletter
\newcommand{\toprule}{\hrule height.8pt depth0pt \kern2pt} 
\newcommand{\midrule}{\kern2pt\hrule\kern2pt} 
\newcommand{\bottomrule}{\kern2pt\hrule\relax}
\newcommand{\algcaption}[2][]{%
  \refstepcounter{algorithm}%
  \toprule
  \textbf{{\raggedright\fname@algorithm~\thealgorithm}}\ #2\par 
  \midrule
}
\makeatother

\def\simplificationName{ALDS\xspace}

\begin{document}
\mainmatter

\title{Rapidly-Exploring Quotient-Space Trees:\\Motion Planning using Sequential Simplifications}
\titlerunning{Rapidly-Exploring Quotient-Space Trees}

\author{Andreas Orthey \and Marc Toussaint}
\authorrunning{A. Orthey and M. Toussaint}

\institute{University of Stuttgart, Germany\\
\email{\{andreas.orthey, marc.toussaint\}@ipvs.uni-stuttgart.de}}

\maketitle
\begin{abstract}

Motion planning problems can be simplified by admissible projections of the configuration space to sequences of lower-dimensional quotient-spaces, called sequential simplifications. To exploit sequential simplifications, we present the Quotient-space Rapidly-exploring Random Trees (QRRT) algorithm. QRRT takes as input a start and a goal configuration, and a sequence of quotient-spaces. The algorithm grows trees on the quotient-spaces both sequentially and simultaneously to guarantee a dense coverage. 
QRRT is shown to be (1) probabilistically complete, and (2) can reduce the runtime by at least one order of magnitude. 
However, we show in experiments that the runtime varies substantially between different quotient-space sequences. To find out why, we perform an additional experiment, showing that the more narrow an environment, the more a quotient-space sequence can reduce runtime.

\end{abstract}

\section{Introduction}

Motion planning algorithms are fundamental for robotic applications like manufacturing, object disassembly, tele-operation or space exploration \cite{lavalle_2006}.
Given an environment and a robot, a motion planning algorithm aims to find a feasible configuration space path from a start to a goal configuration. The complexity of a planning algorithm scales with the degrees-of-freedom (dofs) of the robot, and for high-dof robots the runtime can be prohibitive. 

To reduce the runtime of planning algorithms, prior work has shown that it is effective to consider certain simplifications of the planning problem, such as progressive relaxations \cite{ferbach_1997}, iterative constraint relaxations \cite{bayazit_2005}, lower-dimensional projections \cite{zhang_2009}, possibility graphs
\cite{grey_2017} or quotient space sequences \cite{orthey_2018}. 

We consider a particular class of simplifications which we call \emph{sequential simplifications}. Sequential simplifications are sequences of admissible lower-dimen\-sional simplifications (ALDS). Prior work showed that sequential simplifications can be exploited to construct probabilistically complete algorithms.

Our contribution is an algorithm exploiting sequential simplifications, which we call the Quotient-space Rapidly-exploring Random Trees (QRRT) algorithm. We show QRRT to be probabilistically complete and able to reduce runtime by at least one order of magnitude. 
The runtime reduction, however, depends on the choice of the sequential simplification, which we show to be dependent on narrow passages in the environment.








\section{Related Work}

Prior works on simplifications of the configuration space fall into three categories. First, simplifications which are non-admissible, meaning a simplified solution might not be a necessary condition for a global solution. Second, simplifications which are admissible, but the authors did not develop an algorithm with completeness guarantees, i.e. the algorithm might fail to find a path if one exists. And third, simplifications which are admissible and a (probabilistically) complete algorithm has been developed.


\subsubsection{Non-admissible Simplification}

The runtime of planning algorithms can sometimes be decreased substantially using non-admissible lower-dimensional simplifications. Non-admissible simplifications can be obtained either by random projections \cite{sucan_2009}, or by overestimating the geometry of the robot, either using an enlarged shape \cite{tonneau_2018}, or using balls of free workspace \cite{rickert_2014}. While those works show runtime savings of up to three orders of magnitude \cite{tonneau_2018}, there are no completeness guarantees, and the algorithms might fail even if a path exists.

\subsubsection{Admissible Simplification, Non-Complete Algorithm}

Admissible simplifications have been used in the literature under different names, as sequences of approximations \cite{ferbach_1997}, as iterative constraint relaxations \cite{bayazit_2005}, or as 
sequences of lower-dimensional subproblems \cite{zhang_2009}. All three prior works search sequentially over the abstraction levels, but might fail to find a path because of a lack of backtracking.

\subsubsection{Admissible Simplification, (Probabilistically) Complete Algorithm}

Recent work has shown that probabilistically complete algorithms can be constructed using admissible simplifications by searching not only sequentially over abstraction levels, but also \emph{simultaneously}. This can be done either using a single simplification \cite{grey_2017,gochev_2012}, or sequences of simplifications \cite{orthey_2018}. While those approaches work for arbitrary robots, there exist more efficient algorithms for specialized cases like serial chains \cite{orthey_2018b}, or manipulators \cite{saha_2005}. 
The algorithm we propose uses a sequence of admissible simplifications, similar to \cite{orthey_2018}. However, we differ from \cite{orthey_2018} in (1) developing a single-query planner and (2) evaluating not only one single sequence of simplifications, but multiple sequences of simplification.

\section{Sequential Simplifications}

Configuration spaces can be simplified using a sequence of admissible lower-dimensional simplifications, which we call \emph{sequential simplifications}. To understand sequential simplifications, we first discuss the special case of constraint relaxations \cite{pearl_1984}, and then generalize it to admissible lower-dimensional simplifications (\simplificationName). To design efficient algorithms exploiting \simplificationName, we show that \simplificationName can be seen from three viewpoints. First, they are quotient spaces, spaces where each point represents an equivalence class of configurations. Second, they have a geometrical interpretation as expansion of the free configuration space (Theorem \ref{theorem:lower_dimensional_projection}). Third, they imply admissible heuristics, which allow us to prune equivalence classes of configurations. Consequently, \simplificationName can be seen either as quotient-spaces, or as expansion of free space, or as a source of admissible heuristics.

\subsection{Motion Planning and Constraint Relaxation}

Let $\X$ be the configuration space of a robot. We denote by $\phi:\X \rightarrow \{0,1\}$ the \emph{constraint function} which takes a subset $U \subseteq \X$ and evaluates to zero if at least one $x \in U$ is constraint-free and to one otherwise. The tuple $\planningspace$ will be called a \emph{planning space}.

Let $\Xf = \{x \in \X \mid \phi(x) = 0\}$ be the free configuration space. 
Given an initial configuration $\xi \in \Xf$ and a goal configuration $\xg \in \Xf$, we define $\planningproblem$ to be a \emph{motion planning problem}. 
Our goal is to design an algorithm finding a path from $\xi$ to $\xg$ through $\Xf$.

It is often advantageous to plan in a simplified space. The most basic simplification is a \emph{constraint relaxation}, which is a reduction of $\planningspace$ to $(\X,\tilde{\phi})$ such that
\begin{equation}
    \tilde{\phi}(x) \leq \phi(x)
    \label{eq:cr}
\end{equation}
for all $x\in X$ is fulfilled. Eq.\eqref{eq:cr} is also called a falseness preserving mapping \cite{zhang_2004}. Note that Eq.\eqref{eq:cr} is equivalent to an expansion of the free space as $\Xf \subseteq \tildeXf$ whereby $\tildeXf = \{x \in \X \mid \tilde{\phi}(x) = 0\}$.

\subsection{Admissible Lower-Dimensional Simplifications}

Constraint relaxations are a special case of \emph{admissible lower-dimensional simplifications} (ALDS). A lower-dimensional simplification of $\planningspace$ is a tuple $\alds$ consisting of a projection $\pi: \X \rightarrow \tX$, mapping the configuration space $X$ to a lower-dimensional space $\tX$ and mapping open sets to open sets, together with a constraint function $\tphi: \tX \rightarrow \{0,1\}$ on $\tX$. 

We say that a lower-dimensional simplification is \emph{admissible} if for all $\tx \in \tX$ whenever $\tphi(\tx)=1$ then $\phi(\pi^{-1}(\tx))=1$, or equivalently 
\begin{equation}
    \tphi(\tx) \leq \phi(\pi^{-1}(\tx))
    \label{eq:alds_cr}
\end{equation}
\noindent whereby $\pi^{-1}(\tx)$ is called the \emph{fiber} of $\tx$ in $\X$. We call this an \emph{admissible} lower-dimensional simplification. Constraint relaxation is thus a special case of \simplificationName by using for all $x \in X$ the identity mapping $\pi(x) = \pi^{-1}(x) = x$. 

\subsection{Simplification as Quotient-Space}

Any projection $\pi$ can be viewed as a \emph{quotient space mapping} \cite{munkres_2000}. A quotient space mapping partitions the configuration space into equivalence classes of points $X_{\tx}=\{x\in X \mid \pi(x)=\tx\}$, all mapping to the same point $\tx \in \tX$. The set of equivalence classes is indexed by the space $\tX$, consequently called the quotient space of $X$ under the equivalence relation imposed by $\pi$ \cite{zhang_2004,orthey_2018}. 

\subsubsection{Canonical \simplificationName}

If $X$ is a product space $\X = \tX \times \Z$, we can use a canonical projection $\pi: \X \rightarrow \tX$ to define a canonical \simplificationName. The canonical \simplificationName yields a quotient space $\tX$, where each point $\tx \in \tX$ represents an equivalence class of configurations $\pi^{-1}(\tx)=\{(\tx,z) \mid z \in \Z\}$. To ensure admissibility we define the constraint function $\tphi(\tx)$ to be one if and only if the constraint function $\phi$ of every point in the fiber $\pi^{-1}(\tx)$ evaluates to one. 

\subsubsection{Efficient \simplificationName}
\label{sec:efficient_alds}

A canonical \simplificationName, however, might require to evaluate every member of the fiber $\pi^{-1}(\tx)$ to check admissibility, and therefore fail to be computationally efficient. We say that a simplification is efficient if for any $\tx \in \tX$ there exists an $\xprime$ in the fiber $\pi^{-1}(\tx)$ such that if $\phi(\xprime)=1$ then $\tphi(\tx)=1$. Efficient simplifications can be constructed as nestings of robotic systems, where checking collision of a nested robot implies collision of the original robot, which requires only a \emph{single} collision-check \cite{orthey_2018}.

\subsection{Simplification as Expansion of Free Space}

We can interpret \simplificationName in a geometric way as expansion of free space, as guaranteed by the following theorem

\begin{theorem}
Let $\alds$ be an admissible lower-dimensional projection of the planning space $\planningspace$. Then for all $\tx \in \tX$ the following are equivalent
\begin{align}
    [\tphi(\tx)=1 ]          &\Rightarrow [\phi(\pi^{-1}(\tx))=1]\label{eqn:1}\\
    [\phi(\pi^{-1}(\tx))=0 ] &\Rightarrow [\tphi(\tx)=0]\label{eqn:2}\\
    \Xf                    &\subseteq \pi^{-1}(\tXf)\label{eqn:3}
\end{align}
\label{theorem:lower_dimensional_projection}
\end{theorem}
\begin{proof}
$\eqref{eqn:1} \rightarrow \eqref{eqn:2}$: Contrapositive of \eqref{eqn:1}.

$\eqref{eqn:2} \rightarrow \eqref{eqn:3}$: Let $x \in \Xf$ ($\phi(x)=0$), and let $\tx = \pi(x)$. Then $\tphi(\tx)$ can evaluate to either $0$ or $1$. 
If it is $1$, then by \eqref{eqn:1}, $x$ must evaluate to $\phi(\pi^{-1}(\pi(x)))=1$, a contradiction.
Therefore, $\tphi(\tx)=0$. But then $\phi(\pi^{-1}(\pi(x)))=0$. Therefore $x \in \pi^{-1}(\tXf)$.

$\eqref{eqn:3} \rightarrow \eqref{eqn:1}$: Let $\tx \in \tX$, and $\tphi(\tx)=1$. Then $\tx \in \tXo$. From \eqref{eqn:3} it follows that $\pi^{-1}(\tXo) \subseteq \Xo$. Therefore $\pi^{-1}(\tx) \in \Xo$, and consequently $\phi(\pi^{-1}(\tx))=1$.
$\hfill\qedsymbol$
\end{proof}

\renewcommand\subfigE[2]{
\begin{subfigure}[t]{0.31\linewidth}
   \includegraphics[width=\linewidth]{#1}
   \caption{#2}
\end{subfigure}
}
\begin{figure}[!ht]
    \centering
    \subfigE{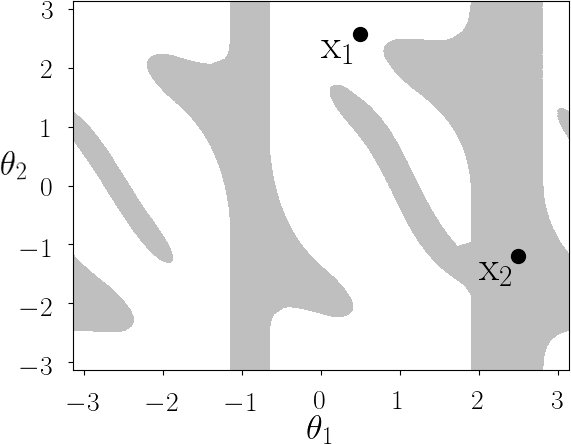}{$\Xf$ (white)}
    \subfigE{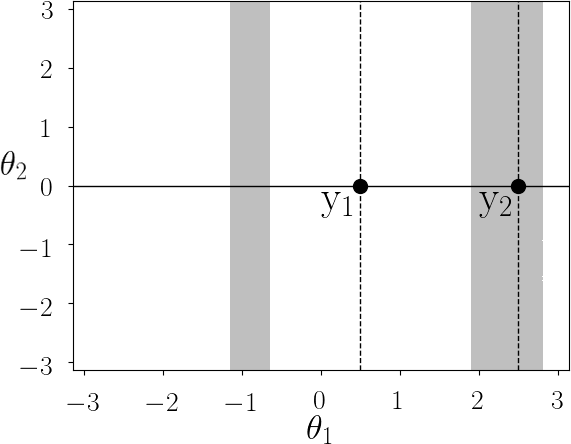}{$\pi^{-1}(\tXf)$ (white)}
    \subfigE{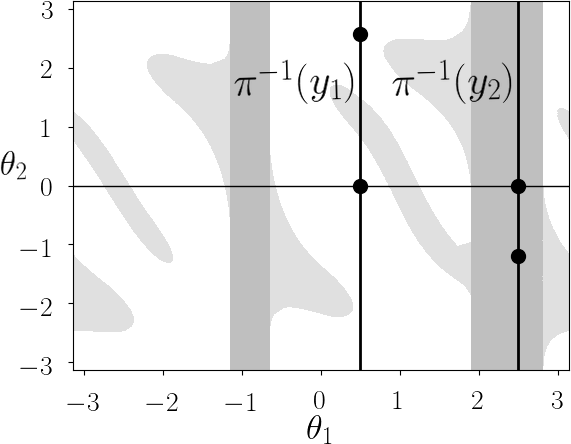}{$\Xf$ (white) as subset of $\pi^{-1}(\tXf)$ (lightgrey plus white)}
    \caption{Illustration of admissible projections as expansion of free space. Figures adapted from Orthey et al. \cite{orthey_2018}.
    \label{fig:dim_reduction_orthey}}
\end{figure}

Fig. \ref{fig:dim_reduction_orthey} shows a visualization of the expansion of free space. We show the configuration space of a 2-dof manipulator in the plane (left), whereby a configuration is marked as gray when it violates constraints. 
One admissible lower-dimensional projection can be obtained by projecting onto the first link by removing the second link. This creates a quotient-space which we can project back into the original configuration space (middle). The overlap is shown on the right, where we observe the free space $\Xf$ to be a subset of the projection of the simplified free space. Theorem \ref{theorem:lower_dimensional_projection} guarantees that this is always true.

We can build admissible lower-dimensional simplifications by nesting robots in each other, by removing degrees-of-freedom \cite{orthey_2018}, removing links \cite{bayazit_2005,zhang_2009}, shrinking links sequentially to zero \cite{baginski_1996} or shrinking links towards the robots medial-axis \cite{saha_2005}. A particular example would be to nest a torso inside a humanoid robot \cite{grey_2017}.
\subsection{Simplification as Source of Admissible Heuristic}

\def\Vpi{W_{\pi(x)}}
\def\U{U}
\def\Ux{\U_x}

An admissible lower dimensional simplification $\alds$ implies an admissible heuristic. This requires us to define in addition (1) a designated goal configuration $x_G$ in $\Xf$, and (2) a cost-to-go function $\hstar$ for every $x \in X$. One cost-to-go function is the \emph{goal-reachability function} defined for all $x$ in $X$ as
\begin{equation}
    \hstar(x) = \begin{cases}
    \infty & \Ux = \emptyset\\
    0 & \text{otherwise}
    \end{cases}
    \label{eq:costtogo}
\end{equation}
\noindent whereby $\Ux$ is a path-connected\footnote{A topological space $\U$ is said to be path-connected if for any $x_1,x_2 \in \U$, there exists a continuous path $p:[0,1] \rightarrow \U$ with $p(0)=x_1$ and $p(1)=x_2$.} subset of $\Xf$ such that $x$ and $x_G$ lie in $\Ux$. If no such subset exists, then $\Ux = \emptyset$. 

Because $\hstar$ is difficult to compute, we search for an estimate $h$ of $\hstar$ which we call a heuristic function \cite{pearl_1984}. One heuristic function is the \emph{lower-dimensional reachability function}
\def\hpi{h_{\pi}}
\begin{equation}
   \hpi(x) = \begin{cases}
    \infty & \Vpi = \emptyset\\
    0 & \text{otherwise}
    \end{cases}
    \label{eq:heuristic}
\end{equation}
\noindent whereby $\Vpi$ is a path-connected subset of $\tXf$ such that $\pi(x)$ and $\pi(x_G)$ lie in $\Vpi$. We like to prove that $\hpi$ is \emph{admissible}, meaning $\hpi(x) \leq \hstar(x)$ for any $x \in X$.
\begin{theorem}
Let $\hstar$ be the goal-reachability function and $\alds$ an admissible lower-dimensional projection. Then the lower dimensional reachability heuristic $\hpi$ is admissible.\label{theorem:admissible_heuristic}
\end{theorem}
\begin{proof}
Let $x\in X$. Either (1) $U_x=\emptyset$, then $\hstar(x)=\infty$, and $\hpi(x) 
\leq \hstar(x)$ by definition. Or (2) there exists $U_x \neq \emptyset$. Then $\hstar(x)=0$. Let $\Vpi = \pi(U_x)$. Then $\Vpi$ is a path-connected subset of $\tXf$, because it is path-connected in $\tX$ by the quotient-space mapping, and it is a subset of $\tXf$ by Theorem \ref{theorem:lower_dimensional_projection}. Further, by definition we have that both $\pi(x)$ and $\pi(x_G)$ lie in $\Vpi$. Therefore $\hpi(x)=0$, and $\hpi(x) \leq \hstar(x)$.
\end{proof}

In Fig.~\ref{fig:dim_reduction_orthey}, the configuration $y_2 = \pi(x_2)$ is infeasible on the lower-dimensional space, i.e. $\hpi(x_2)$ is infinite. Because $h(x_2)$ is an underestimate, we know that the fiber $\pi^{-1}(y_2)$ (solid vertical line) is infeasible (i.e. $\hstar(\xprime)=\infty$ for any $\xprime \in \pi^{-1}(y_2)$) and can be ignored. This is sometimes referred to as the \emph{pruning power} of an admissible heuristic \cite{pearl_1984}.

\subsection{Sequential Simplifications}

We define a sequential simplification of a planning space $\planningspace$ to be a sequence of ALDS, that is, a sequence of $K$ quotient-spaces $\{X_1,\cdots,X_K\}$ obtained from $K-1$ admissible lower-dimensional simplifications $\{\pi_k,\tilde{\phi}_k\}_{k=1}^{K-1}$ 
such that $X_K = X$ and $\phi_K = \phi$. For $K=1$, we obtain the original planning space $\planningspace$, which we call the \emph{trivial simplification}.

\section{Rapidly-exploring Random Quotient-Space Trees}

Our goal is to develop an algorithm which solves the motion planning problem $\planningproblem$ by exploiting sequential simplifications. We require an additional input of quotient-spaces $\{X_1, \cdots, X_K\}$ with $X_K=\X$. The quotient-spaces are then exploited to quickly prune equivalence classes of configurations which are infeasible on a lower-dimensional quotient-space.

We introduce the Quotient-space Rapidly-exploring Random Trees (QRRT) algorithm, which generalizes the Rapidly-exploring Random Tree (RRT) algorithm \cite{lavalle_1998} by growing a tree on each quotient-space instead of growing a single tree on the configuration space.

We divide the description of QRRT into three parts: First we discuss the RRT algorithm, second we describe QRRT as the generalization of RRT to sequences of quotient-spaces, and third we prove that QRRT is probabilistically complete.

\subsection{Rapidly-exploring random trees}

Algorithm \ref{alg:rrt} describes the classical RRT algorithm \cite{lavalle_1998}. RRT takes as input a configuration space $X$, an initial and goal configuration $x_I,x_G \in X$ and returns a path between $x_I$ and $x_G$ if one exists. The algorithms iteratively grows a single tree while a \emph{planner terminate condition} (PTC) is false (Line 2). The PTC can be a number of iterations, a time limit or a desired cost.

In each iteration, we grow the tree (Line 3), and if there is a connection along the tree between $x_I$ and $x_G$ (Line 4), the algorithm returns the shortest path on the tree (Line 5). The grow function is depicted in Algorithm \ref{alg:rrt_grow}, in which we first sample a random point (Line 1), compute the nearest point on the tree (Line 2), and then connect the nearest point to the random point (Line 3).

\subsection{Rapidly-exploring random quotient-space trees}

QRRT generalizes RRT by growing quotient-space trees both sequentially and simultaneously. A description of QRRT is given in Algorithm \ref{alg:qrrt}. Inside the algorithm, we use a priority queue (Line 1) to sort all quotient-spaces according to the number of nodes in each tree. The space with the highest priority is the one with the least number of nodes, calculated as $\frac{1}{N+1}$ with $N$ the number of nodes in the tree $\G_k$. 
Let us assume that we are at iteration $k$ of the for loop (Line 2). From $X_k$ we create a quotient class $Q_k$ (Line 3). The quotient class $Q_k$ consists (1) of a tree $\G_k$ initialized with $x^k_I$, (2) a reference to the previous quotient-space tree $\G_{k-1}$ (if any), and (3) of $C_k=\bigslant{X_k}{X_{k-1}}$ which we refer to as the \emph{fiber space}.
We then push $Q_k$ into the queue (Line 4), which contains all quotient classes $Q_1,\cdots,Q_k$. 
We then iterate until a path has been found on $X_k$ or the PTC becomes true (Line 5), by first taking the quotient-space $Q_{\text{least}}$ from the priority queue with the minimum number of samples (Line 6), grow this quotient-space (Line 7), and push it back onto the queue (Line 8). If the quotient-space $Q_{\text{least}}$ is successfully expanded, the priority will be decreased, so that in the limit every quotient-space will be expanded infinitely many times. 
We then take the quotient-space $Q_k$, check if it contains a connection between initial and goal configuration (Line 9), and if yes compute a path (Line 10). Once a path has been computed, the while loop becomes false (Line 5), and we move to the next quotient-space $X_{k+1}$, or terminate and return the path if $k=K$. 

The growing of the quotient tree (Line 7) is further detailed in Algorithm \ref{alg:qrrt_grow}. GrowQRRT differs from GrowRRT only by its sampling routine: Instead of sampling $X_k$ directly, we sample first a random configuration $x_{k-1}$ from the tree $\G_{k-1}$, and another a configuration $x_c$ from the fiber space. Both configurations together define a unique configuration $\xr$ in $X_k$, with $x_{k-1}$ indexing the fiber, and $x_c$ the position inside the fiber. (Line 1). 
This sampling routine is dense in the \emph{free} configuration space, and the algorithm is therefore probabilistically complete (see below for a proof). Many sampling variants are possible. In our implementation, we sample a random vertex from $\G_{k-1}$, but we could also sample along the edges, sample a neighborhood, or even bias sampling towards a shortest path \cite{orthey_2018}. 

For $K=1$, the QRRT becomes (almost) equivalent to RRT. The difference is an additional overhead of pop-ing (Line 6) and push-ing (Line 8) the configuration space in and out of the priority queue (compare to Algorithm \ref{alg:rrt}). In some preliminary experiments, we observed essentially no difference (on average) between RRT and QRRT with $K=1$. We will refer to this case as QRRT with trivial simplification.

Our algorithm is freely available as a C++ implementation, and has been submitted to the Open Motion Planning Library (OMPL) \cite{sucan_2012}.

\lstset{language=C}
\alglanguage{pseudocode}
\begin{figure}
    \centering
    \algcaption{RRT($x_I,x_G,\X$)\cite{lavalle_1998}}
    \begin{algorithmic}[1]
      \State $\G=\Call{Tree}{x_I}$
      \While{$\neg\Call{ptc}{}$}
        \State $\Call{GrowRRT}{\G,\X}$
        \If{$\Call{IsConnected}{x_I,x_G,\G}$}
        \State \Return \Call{Path}{$x_I,x_G,\G$}
        \EndIf
      \EndWhile
     \State \Return $\emptyset$
    \end{algorithmic}
    \label{alg:rrt}
  \bottomrule
  \bigskip
  \algcaption{GrowRRT($\G,\X$)}
  \begin{algorithmic}[1]
    \State $\xr \gets \Call{Sample}{\X}$
    \State $\xn \gets \Call{Nearest}{\xr,\G}$
     \State $\xw \gets \Call{Connect}{\xn,\xr,\G}$
  \end{algorithmic}
  \label{alg:rrt_grow}
  \bottomrule
  \bigskip
    \centering
    \algcaption{QRRT($x^1_I,\ldots,x^K_I,
    x^1_G,\ldots,x^K_G,
    X_1,\cdots,X_K$)}
\begin{algorithmic}[1]
  \State $\PriorityQueue \gets \Call{priority\_queue}{}$
  \For{$k=1$ to $K$}
    \State $Q_k=\Call{InitQuotient}{x^k_I, x^k_G, X_k,Q_{k-1}}$
    \def\Qbest{Q_{\text{least}}}
    \State $\PriorityQueue.\Call{push}{Q_k}$
    \While{$\path_k == \emptyset$ and $\neg\Call{ptc}{}$}
      \State $\Qbest = \PriorityQueue.\Call{pop}{}$
      \State $\Call{GrowQRRT}{\Qbest}$
      \State $\PriorityQueue.\Call{push}{\Qbest}$
      \If{$\Call{IsConnected}{Q_k}$}
      	\State $\path_k = \Call{Path}{Q_k}$
      \EndIf
    \EndWhile
  \EndFor
\State \Return $\path_K$
\end{algorithmic}\label{alg:qrrt}
\bottomrule
\bigskip
\algcaption{InitQuotient($x^k_I,x^k_G,X_k,Q_{k-1}$)}\label{alg:qmpinit}
\begin{algorithmic}[1]
  \State \Return $\{X_k, 
  x^k_G,
  \G_k=\Call{Tree}{x^k_I},
  \G_{k-1},
  C_k=\bigslant{X_k}{X_{k-1}}
   \}$
\end{algorithmic}
\label{alg:qrrt_init}

\bottomrule
\bigskip

  \algcaption{GrowQRRT($Q_k$)}
  \begin{algorithmic}[1]
    \State $\xr \gets \Call{SampleQuotient}{\G_{k-1}} \circ \Call{Sample}{C_k}$
    \State $\xn \gets \Call{Nearest}{\xr,\G_k}$
     \State $\xw \gets \Call{Connect}{\xn,\xr,\G_k}$
  \end{algorithmic}
  \label{alg:qrrt_grow}
\bottomrule
\end{figure}
\subsection{Probabilistic Completeness}

The QRRT algorithm is probabilistically complete. A motion planning algorithm is probabilistically complete, if the probability that the algorithm will find a path (if one exists) goes to one as the number of samples goes to infinity. This property has been proven for sampling-based planners \cite{svestka_1996}, in particular RRT \cite{kuffner_2000}.

For sampling-based planners, probabilistic completeness is often proven using a two-step method \cite{svestka_1996}. First, we show that every open set in the configuration space is sampled at least once. Second, we show any feasible paths will be found by the series-of-balls argument \cite{svestka_1996,berenson_2011}, where a feasible path is covered by a series of overlapping open sets, and we show the tree can be expanded along them by sampling in the intersection of the open sets.

This two-step method needs to be generalized to quotient-space sequences. Our approach is to change the first step, by showing that any open set in the \emph{free} $K$-th quotient-space will be sampled at least once. Completeness directly follows by the series-of-balls argument.

To prove that any open set in $X_{K,\free}$ is sampled at least once, we show the sampling sequence of Algorithm \ref{alg:qrrt_grow} is dense\footnote{A set $S$ is dense in $X$ if for any non-empty open subset $V$ of $X$, the intersection $S \cap V$ is non-empty \cite{munkres_2000}.}.
Let $\alpha_k=\{\alpha^n_k, n\in \N\}$ be a sampling sequence on the $k$-th quotient-space. Algorithm \ref{alg:qrrt_grow} defines $\alpha_k$ such that $\alpha_k$ is dense in $\pi^{-1}_{k}(\alpha_{k-1})$, and $\alpha_1$ is dense in $X_1$. From this definition we like to show that $\alpha_K$ is dense in $X_{K,\free}$. 

\begin{theorem}
$\alpha_K$ is dense in $X_{K,\free}$ for $K \geq 1$.
\label{theorem:dense_sampling}
\end{theorem}

\begin{proof}
By induction for $K=1$, $\alpha_K$ is dense in $\X_1$ by definition, and therefore dense in $X_{1,\free}\subseteq \X_1$. Assume $\alpha_{K-1}$ is dense in $X_{K-1,\free}$. Let $V$ be a non-empty open subset of $X_{K,\free}$. 
Since $V$ is open, $\pi_{K-1}(V)$ is open. By induction assumption there exists a $y$ in $\alpha_{K-1} \cap \pi_{K-1}(V)$.
Consider an open set $M$ of the fiber $\pi^{-1}_{K-1}(y)$ (Note $M$ might be closed in $X_{K,\free}$). Since $\alpha_K$ is dense in $\pi^{-1}_{K-1}(\alpha_{K-1})$, there exists an $x$ in $\alpha_K \cap M$ which is a subset of $V$. Since $V$ was arbitrary, $\alpha_K$ is dense in $X_{K,\free}$.
\end{proof}




\section{Experiments}

We evaluate QRRT in two parts. First, we take four robotic systems and compare how QRRT performs using different sequential simplifications. Second, we investigate under which circumstances sequential simplifications outperform the trivial simplification. 

\subsection{QRRT with Efficient Sequential Simplifications\label{sec:experiment_1}}

We evaluate the runtime of QRRT using different sequential simplifications on four robotic systems. We use a fixed-base planar 8-dof manipulator robot, a free-floating planar 8-dof serial linkage, a fixed-base spatial 7-dof KUKA manipulator, and a free-floating spatial 10-dof serial linkage. 

For each robot, we define a set of $J \in \N$ admissible projections. Those projections can be combined to obtain $N(J)$ different sequential simplifications. The number $N(J)$ is given by all combinations of the set $S=\{1,\cdots,J\}$, and can be computed as

\begin{equation}
    N(J) = \sum\limits_{j=0}^{J} \binom{J}{j}
    \label{eq:binom}
\end{equation}

\noindent which is a sum of binomial coefficients known to be equivalent to $2^{J}$.

\subsubsection{Fixed-Base Planar}

In our first experiment, we consider an 8-dof planar manipulator with configuration space $\R^8$. 
We define seven efficient projections onto quotient-spaces $\R^{\{1,\cdots,7\}}$, each obtained by removing links from the original robot. This is visualized in Fig. \ref{fig:results_planar_manipulator_subspaces}, where we show the body of the robot at the projected start (green) and projected goal (red) configuration. Note that only removing end links from the robot results in an \emph{efficient} ALDS, whereas removing intermediate links results in a canonical ALDS for which the admissible constraint function $\phi_Y(y)$ in general requires to evaluate every member of the fiber $\pi^{-1}(y)$. The number of sequential simplifications obtained from the seven projections is $2^7=128$. 

\renewcommand\subfigE[2]{
\begin{subfigure}[b]{0.23\linewidth}
   \includegraphics[width=\linewidth]{#1}
   \caption{#2}
\end{subfigure}
}
\begin{figure}[!ht]
    \centering
    \subfigE{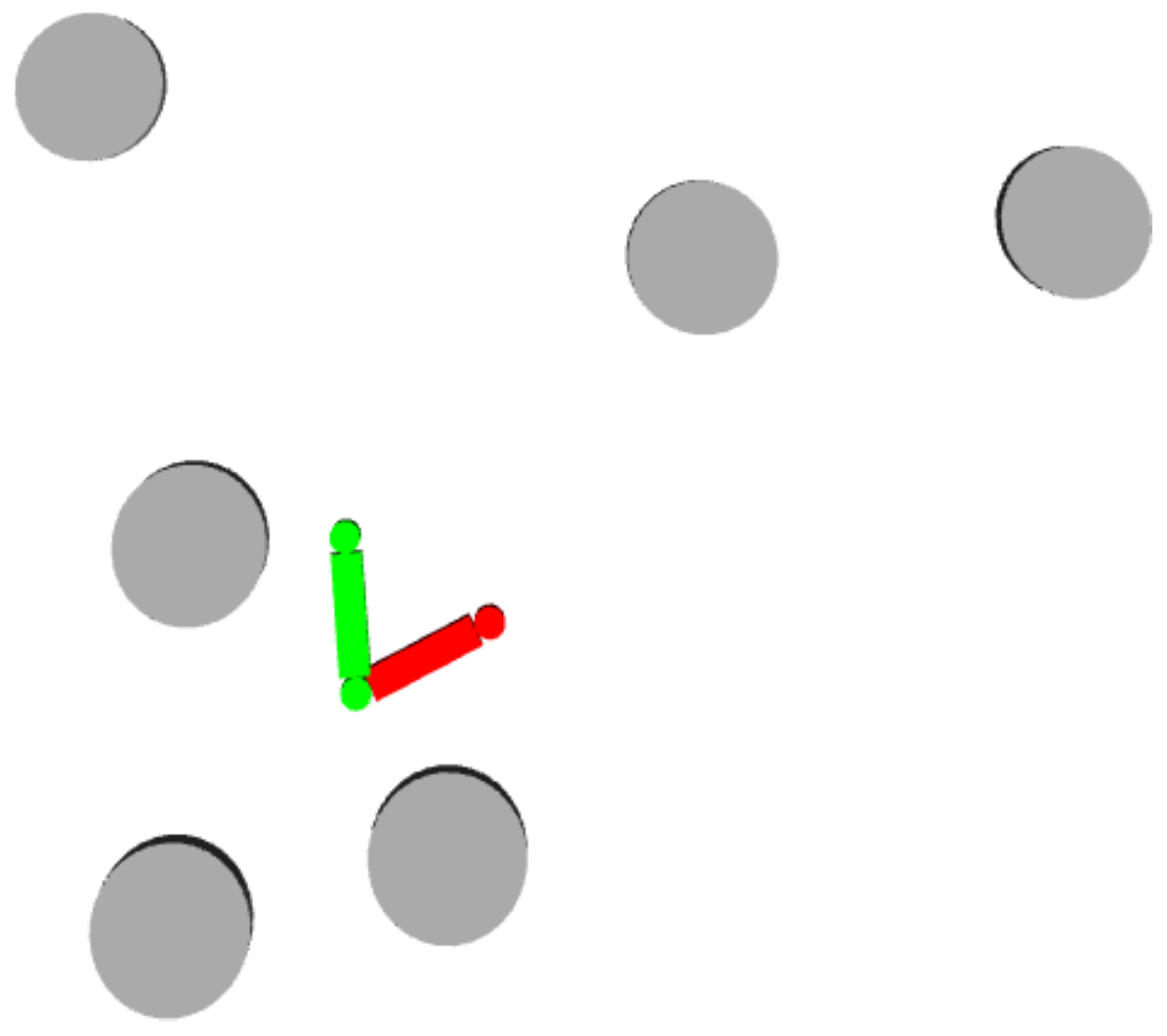}{$\R^1$}
    \subfigE{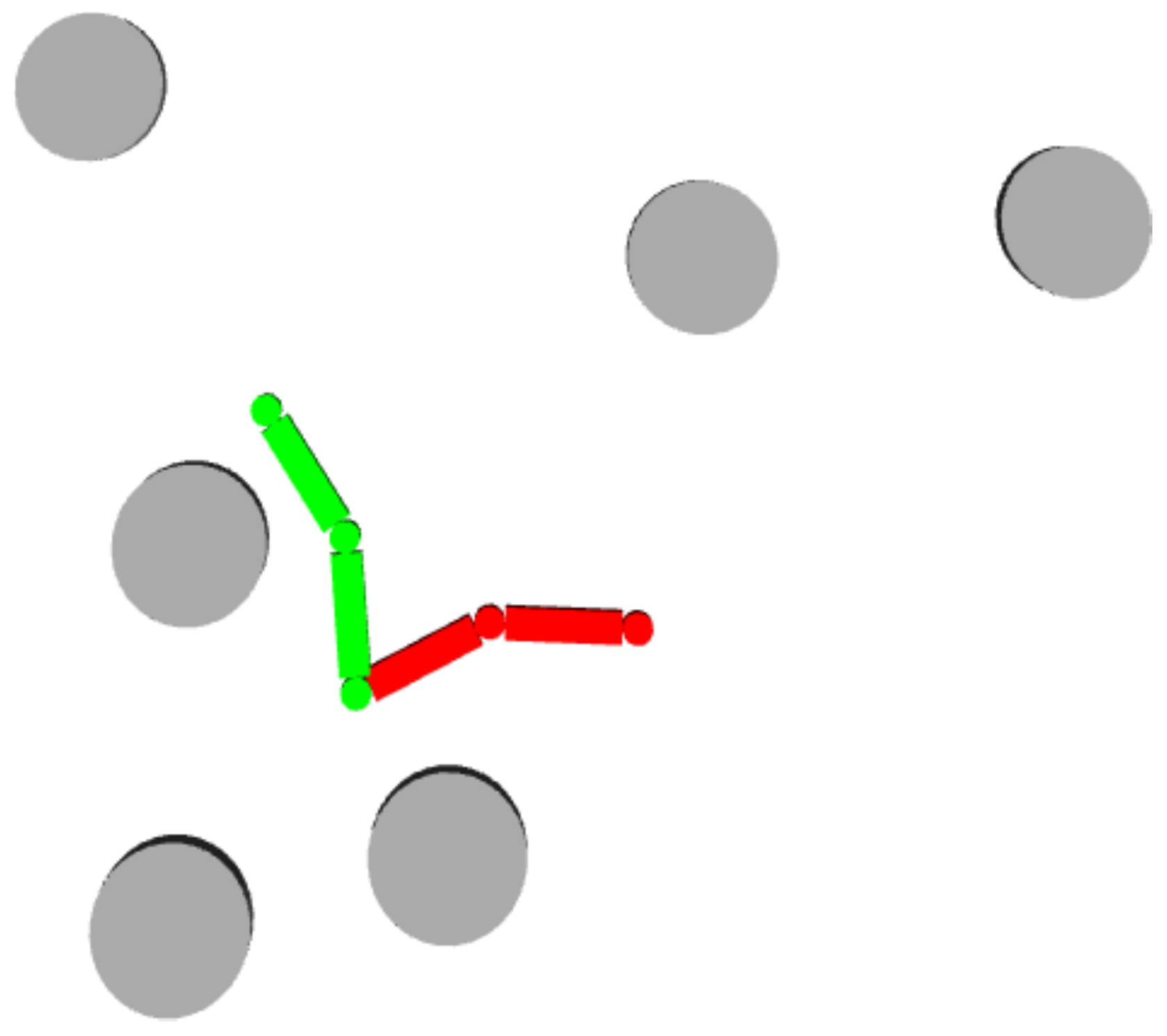}{$\R^2$}
    \subfigE{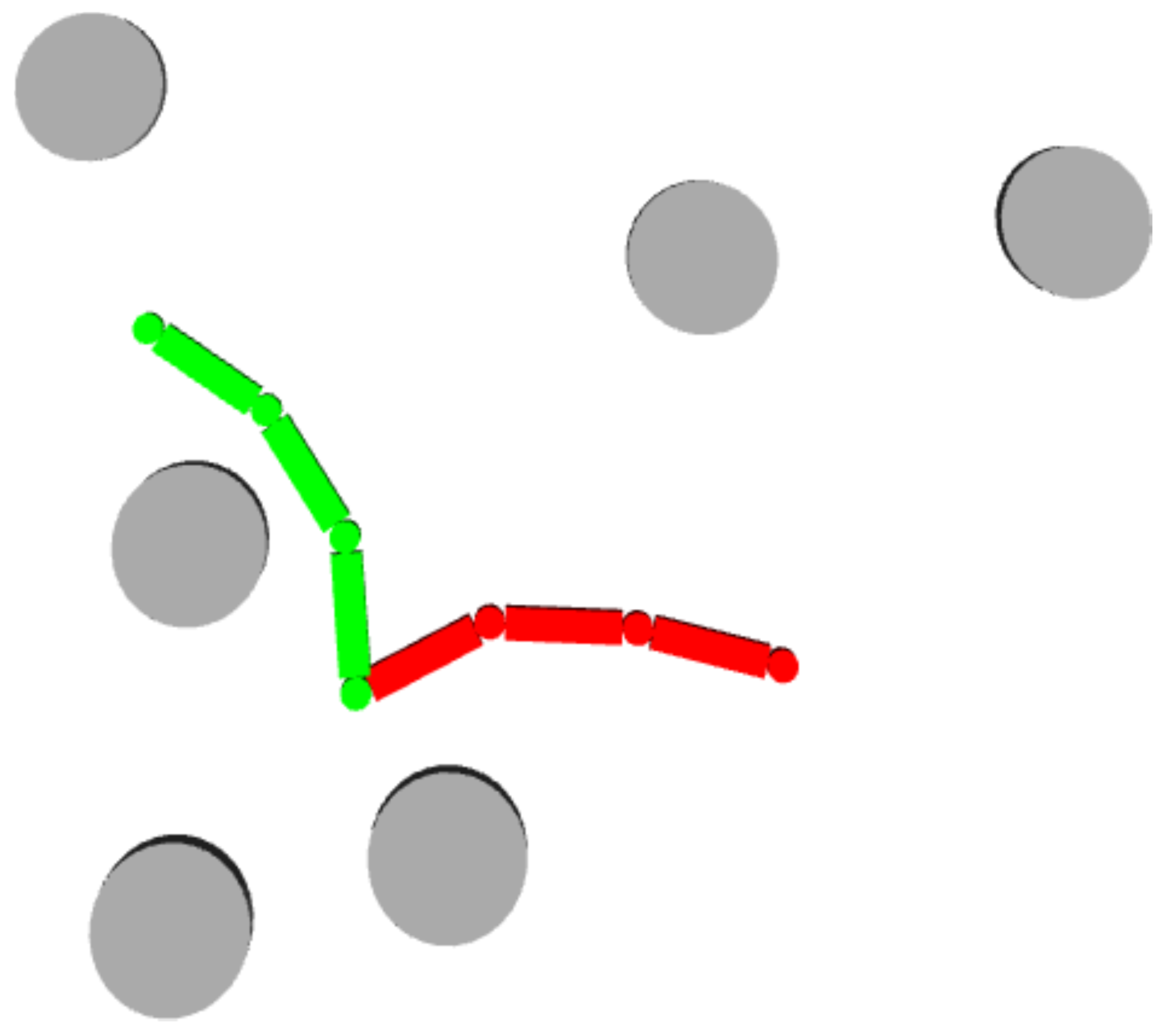}{$\R^3$}
    \subfigE{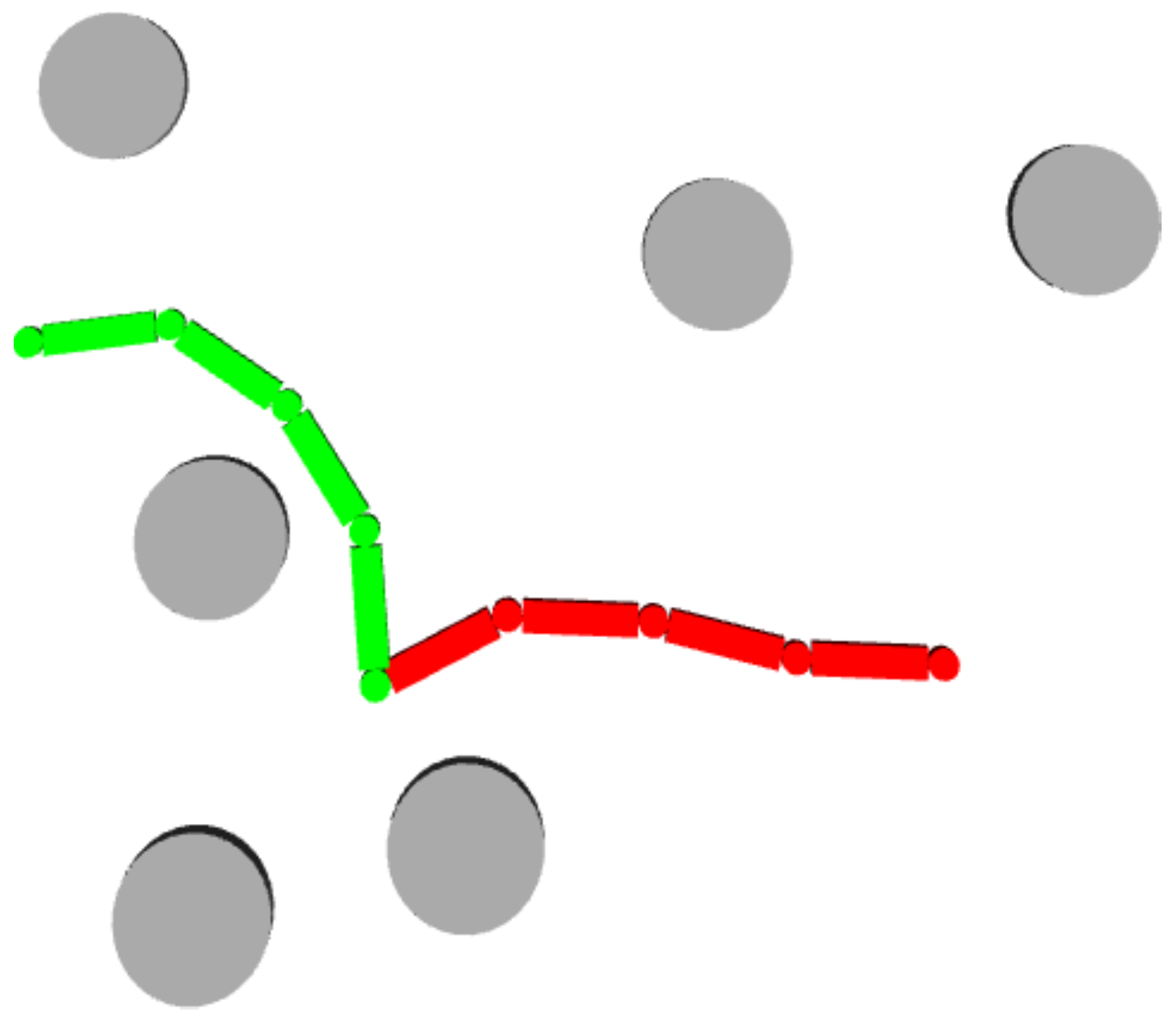}{$\R^4$}
    \subfigE{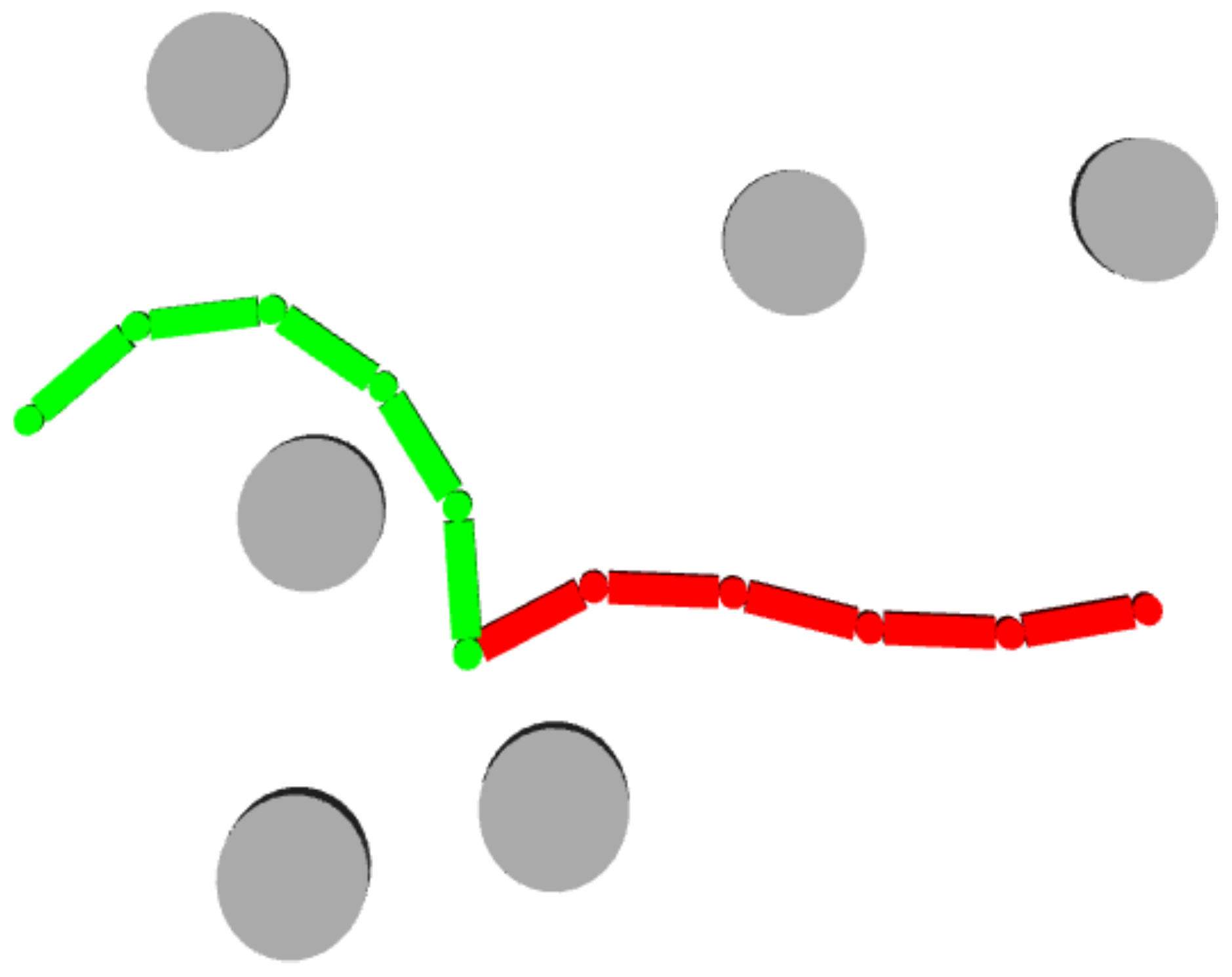}{$\R^5$}
    \subfigE{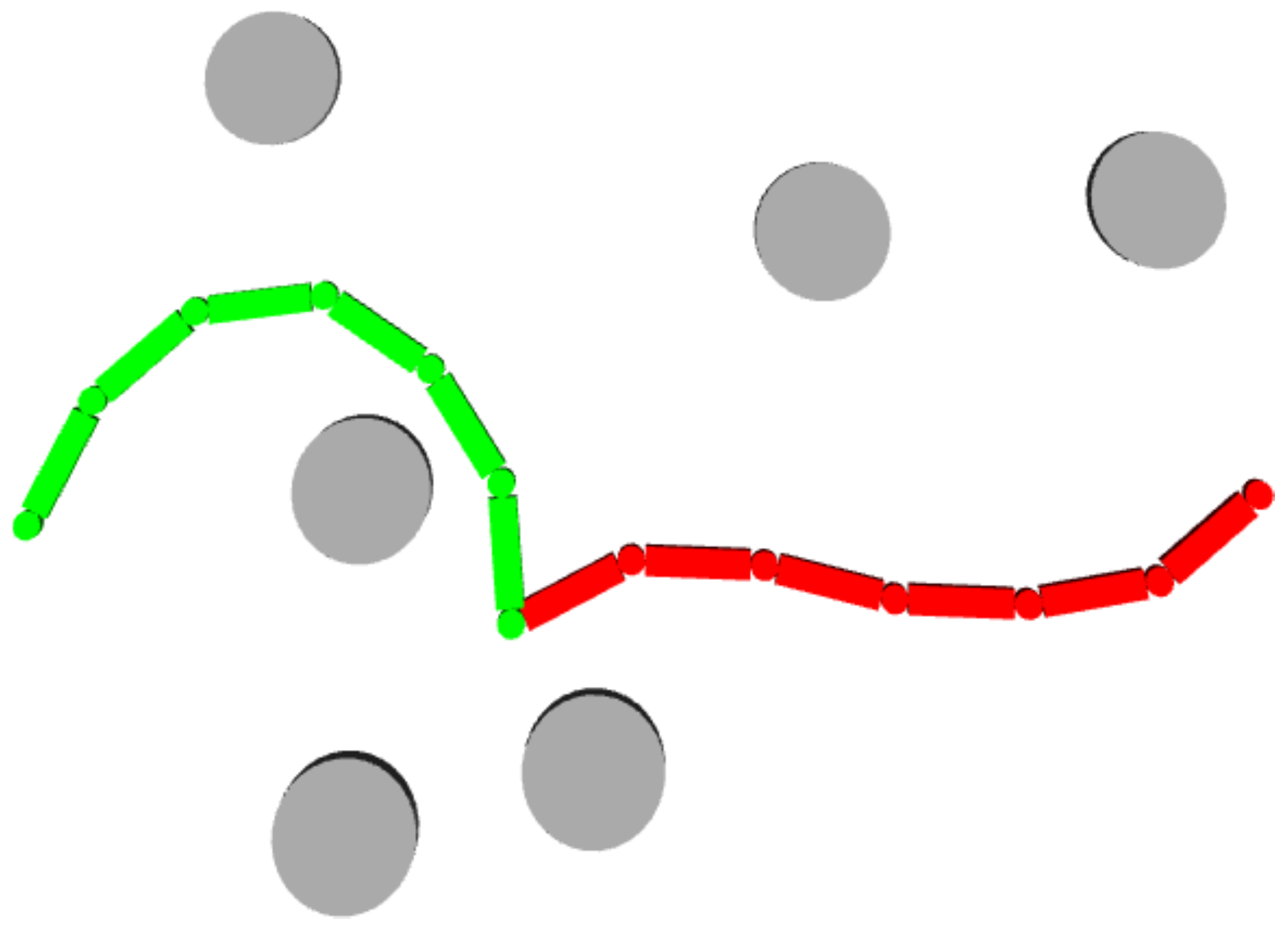}{$\R^6$}
    \subfigE{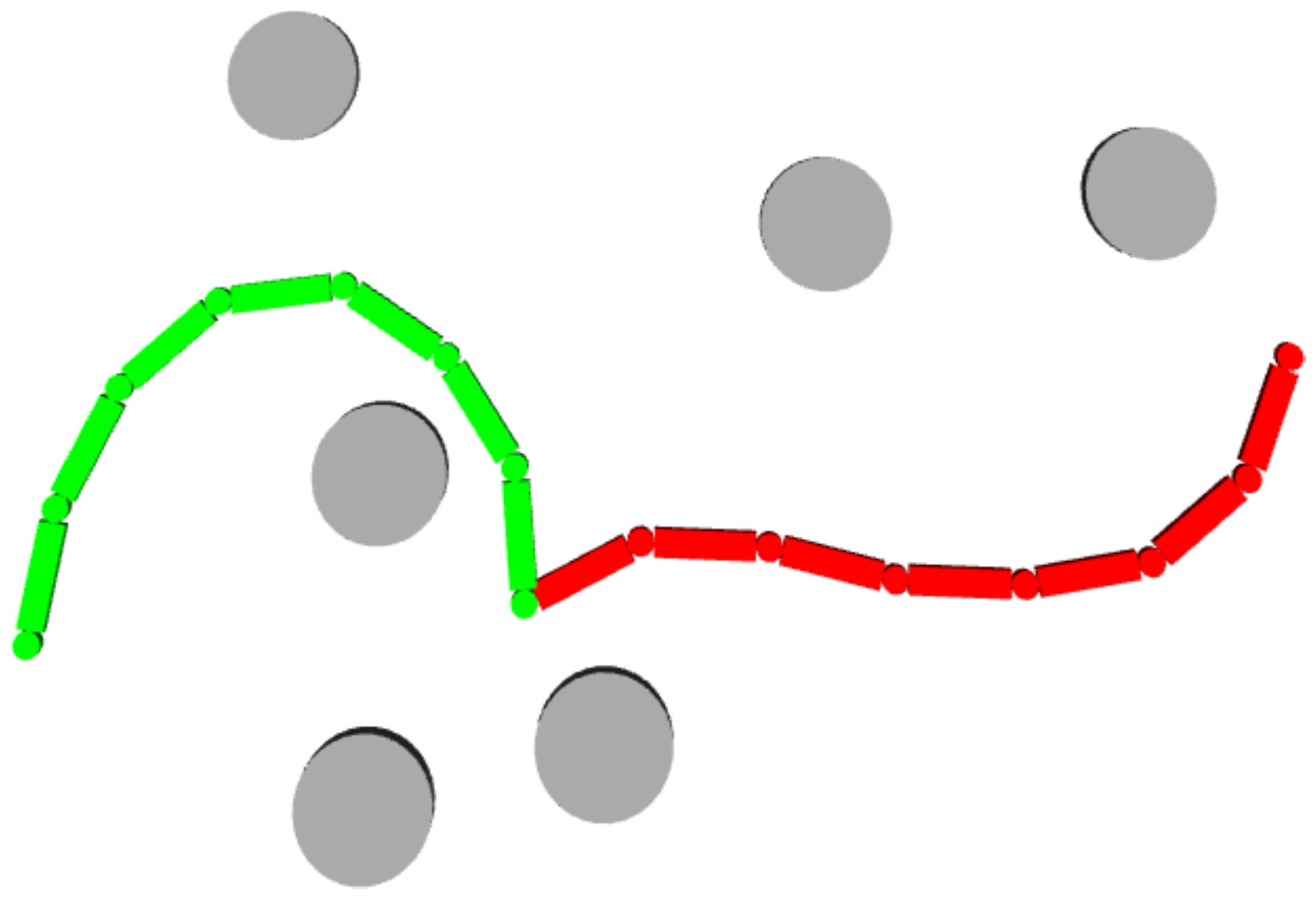}{$\R^7$}
    \subfigE{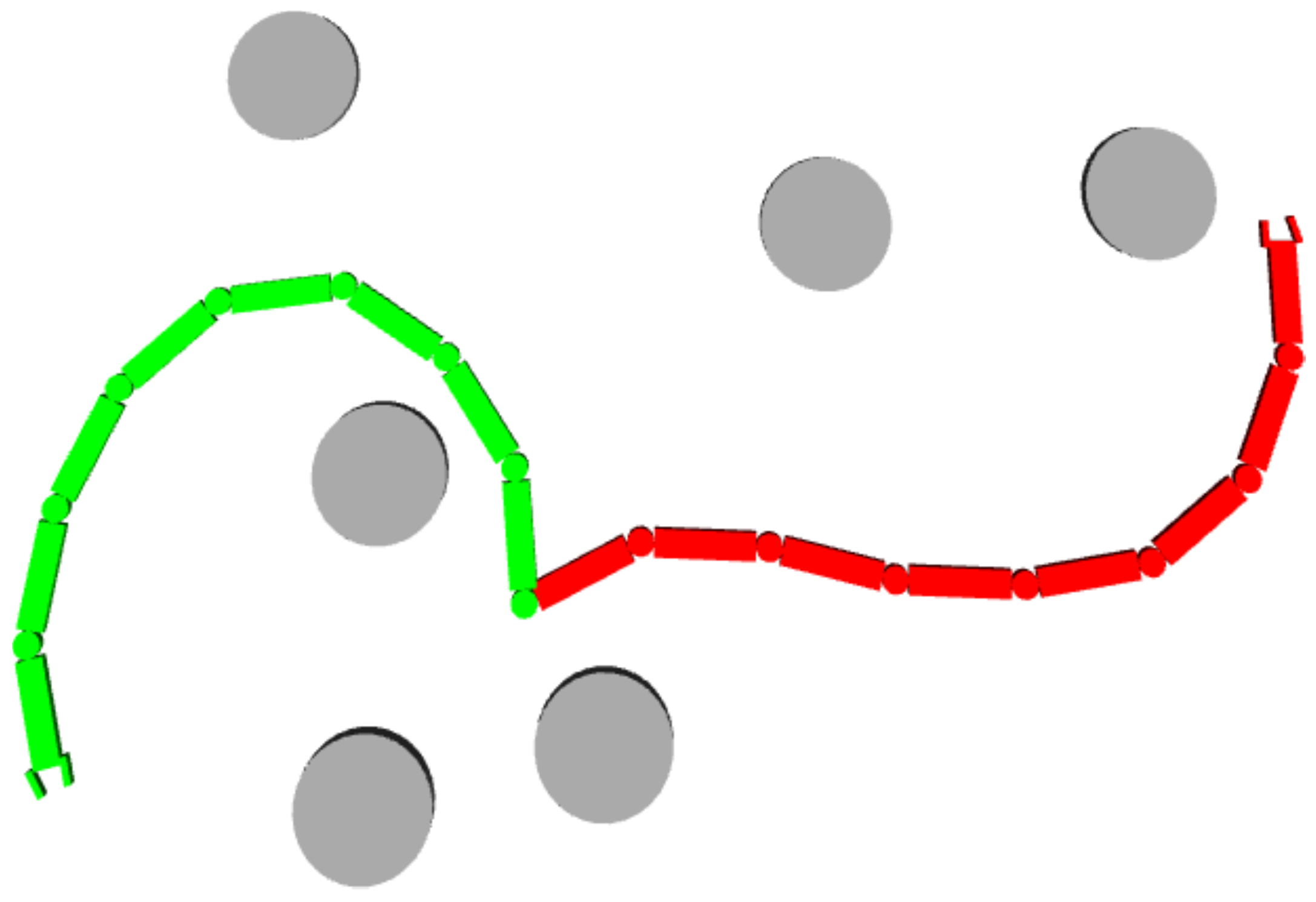}{$\R^8$}    
    
    \caption{The start (green) and goal (red) configuration of the 8-dof manipulator projected onto seven quotient-spaces $\R^{\{1,\cdots,7\}}$, and onto the configuration space $\R^8$.}
    \label{fig:results_planar_manipulator_subspaces}
\end{figure}

\begin{figure}[!ht]
    \centering
    \includegraphics[width=0.48\linewidth]{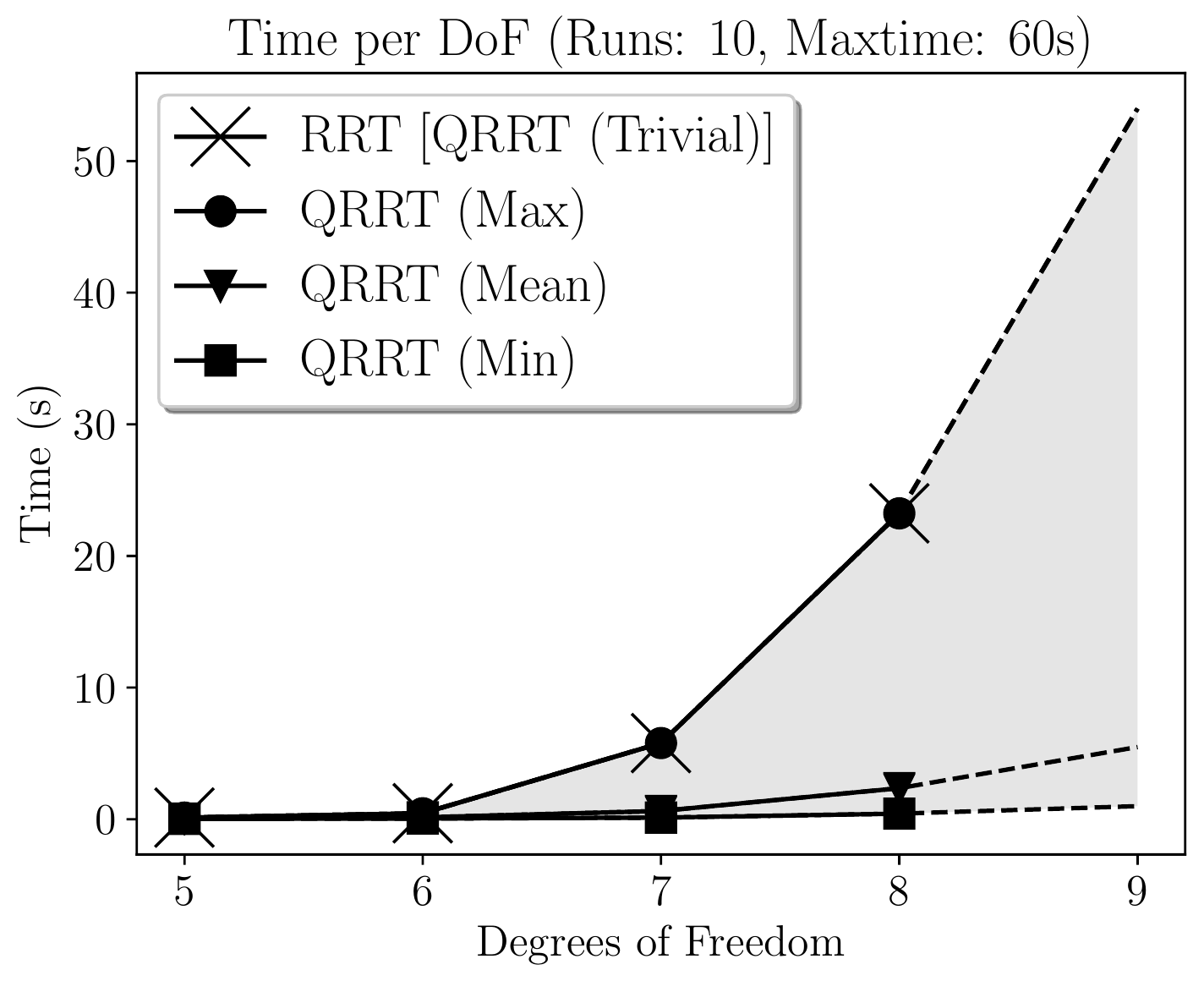}
    \includegraphics[width=0.48\linewidth]{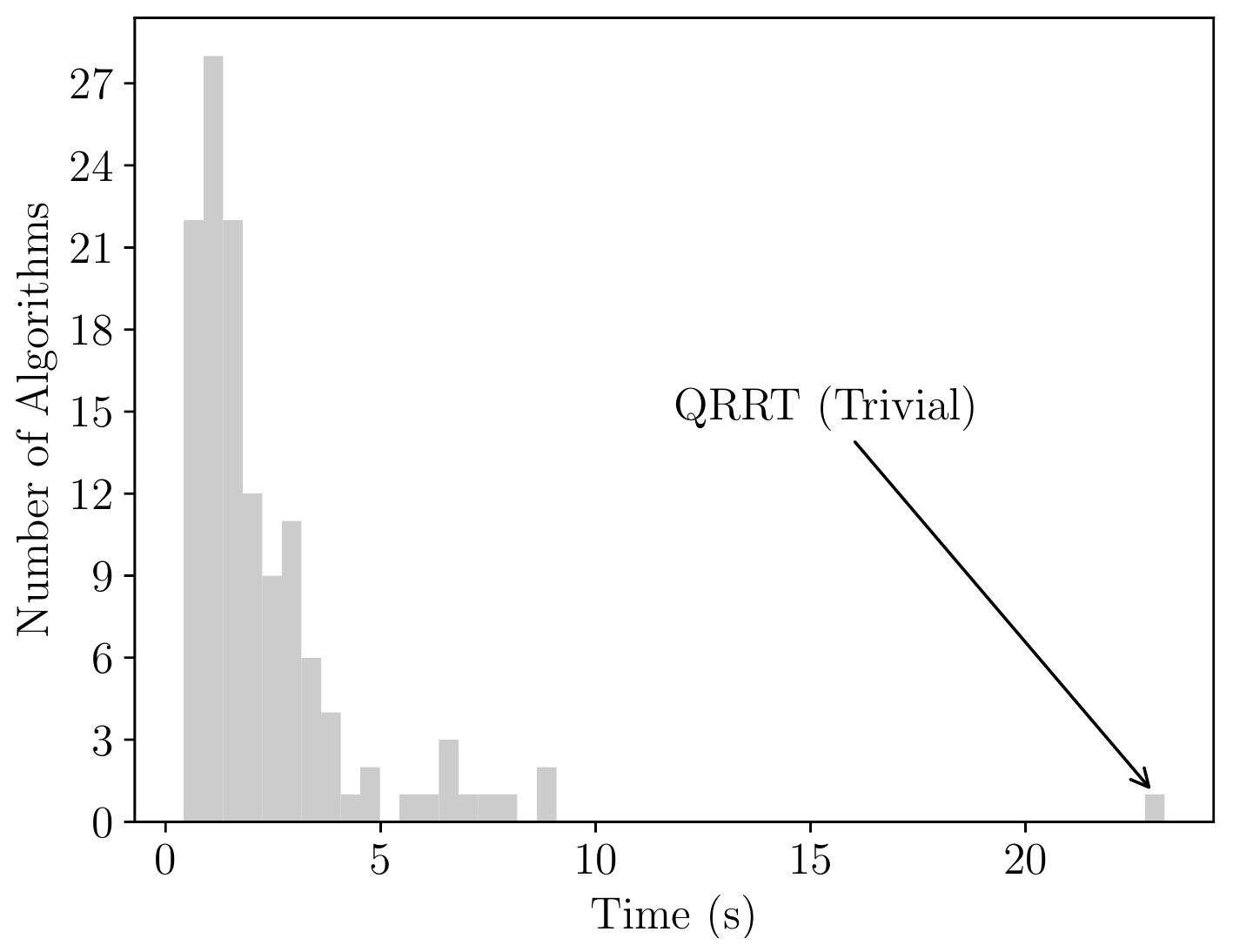}
    \caption{{\bf Left:}Runtime of QRRT on the planar manipulator scenario. For each number of DoFs, we enumerated all possible sequential simplifications, let QRRT run 10 times, and average the runtime. We then show the minimum, the maximum, the mean and the runtime of the trivial simplification (RRT).{\bf Right:} For the $8$ DoF case, we count the number of simplifications having similar runtime. See text for clarification.}
    \label{fig:results_planar_manipulator}
\end{figure}

We then run QRRT with each sequential simplification and compare the runtime. The runtime is obtained by averaging over $10$ runs with a time limit of $60$s. To observe the runtime when the dofs increase, we additionally run the algorithm on a $5$,$6$, and $7$-dof version of the problem, where we have $16,32$ and $64$ sequences, respectively. The outcome is shown in Fig. \ref{fig:results_planar_manipulator} (left). Over all the QRRT runtimes, we compute the minimum runtime (rectangle marker), the maximum (disk), the mean (triangle), and the runtime of the trivial simplification (cross). Note that the trivial simplification is equivalent to RRT. We observe that the algorithm using the trivial simplification performs worst, while at least one algorithm achieves a runtime of less than one second, an improvement of one order of magnitude.

In Fig. \ref{fig:results_planar_manipulator}(right) we visualize, for the $8$-dof case, the distribution of all different sequential simplifications. One can see that the majority of sequential simplifications has a runtime of less than $5$s, while the trivial simplification is at $24$s.

\subsubsection{Free-Floating Planar}

In the second experiment, we use a $7$-dof planar free-floating robot with configuration space $SE(2) \times \R^4$, and we define five efficient quotient-space projections as shown in Fig. \ref{fig:results_planar_snake_subspaces}. The results are shown in Fig. \ref{fig:results_planar_snake} (left) showing the trivial simplification to perform best with a runtime of around $1$ second for $7$ dofs. The distribution of simplifications is shown in Fig. \ref{fig:results_planar_snake} (right). In this case, the trivial simplification performed best with $1.5$s, compared to the worst with around $18$s.

\renewcommand\subfigE[2]{
\begin{subfigure}[b]{0.31\linewidth}
   \includegraphics[width=\linewidth]{#1}
   \caption{#2}
\end{subfigure}
}
\begin{figure}[!ht]
    \centering
    \subfigE{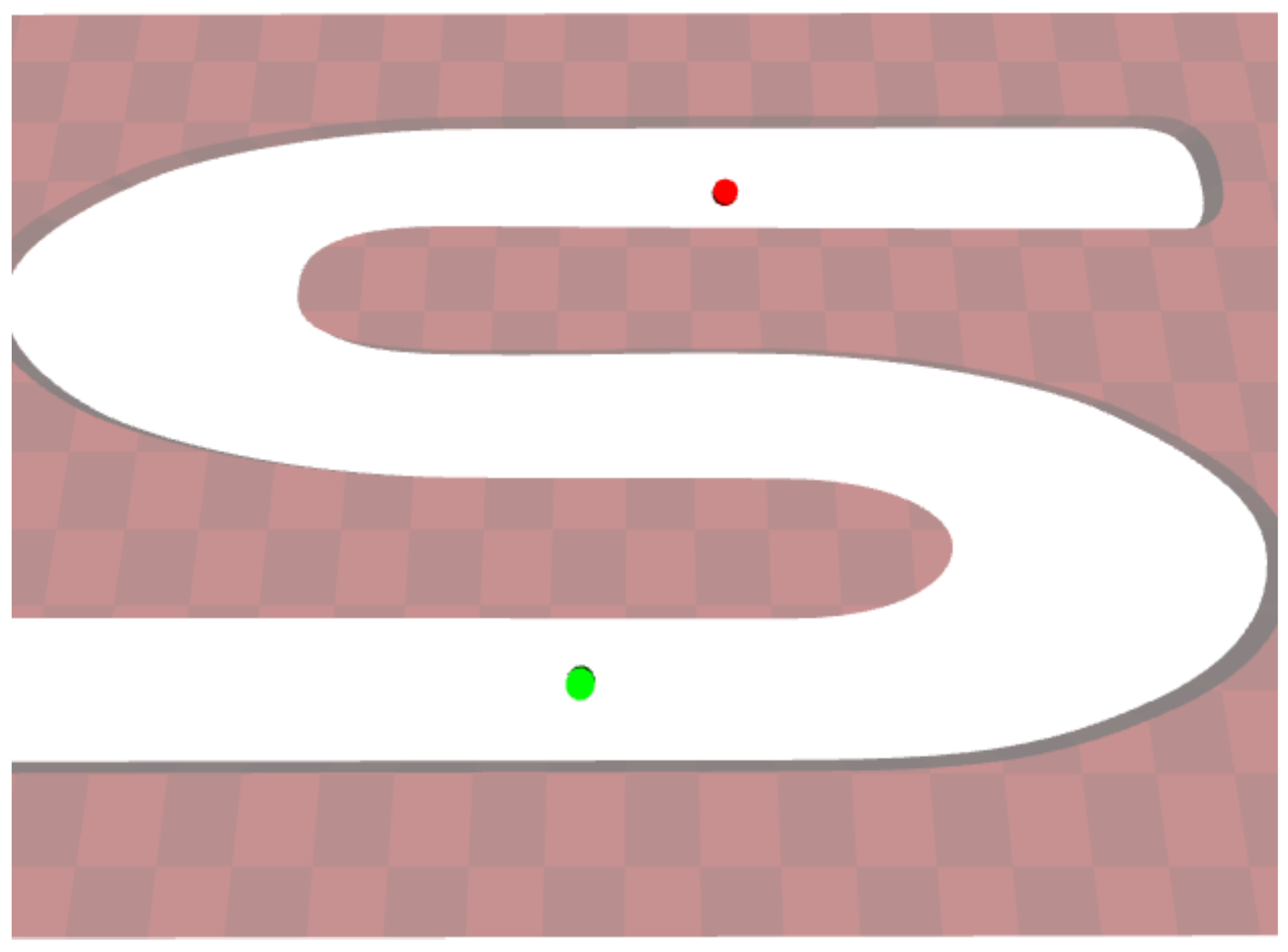}{$\R^2$}
    \subfigE{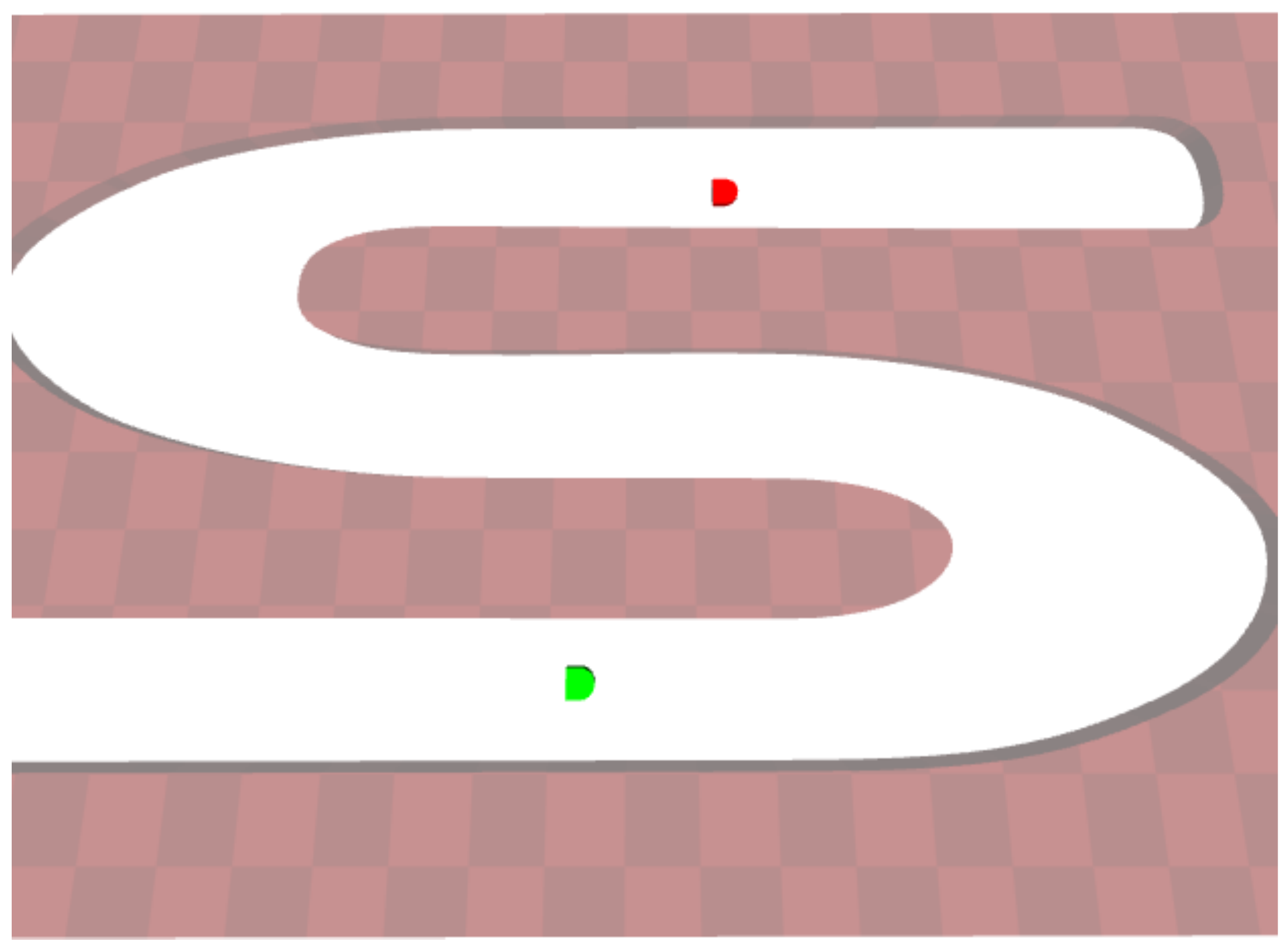}{$SE(2)$}
    \subfigE{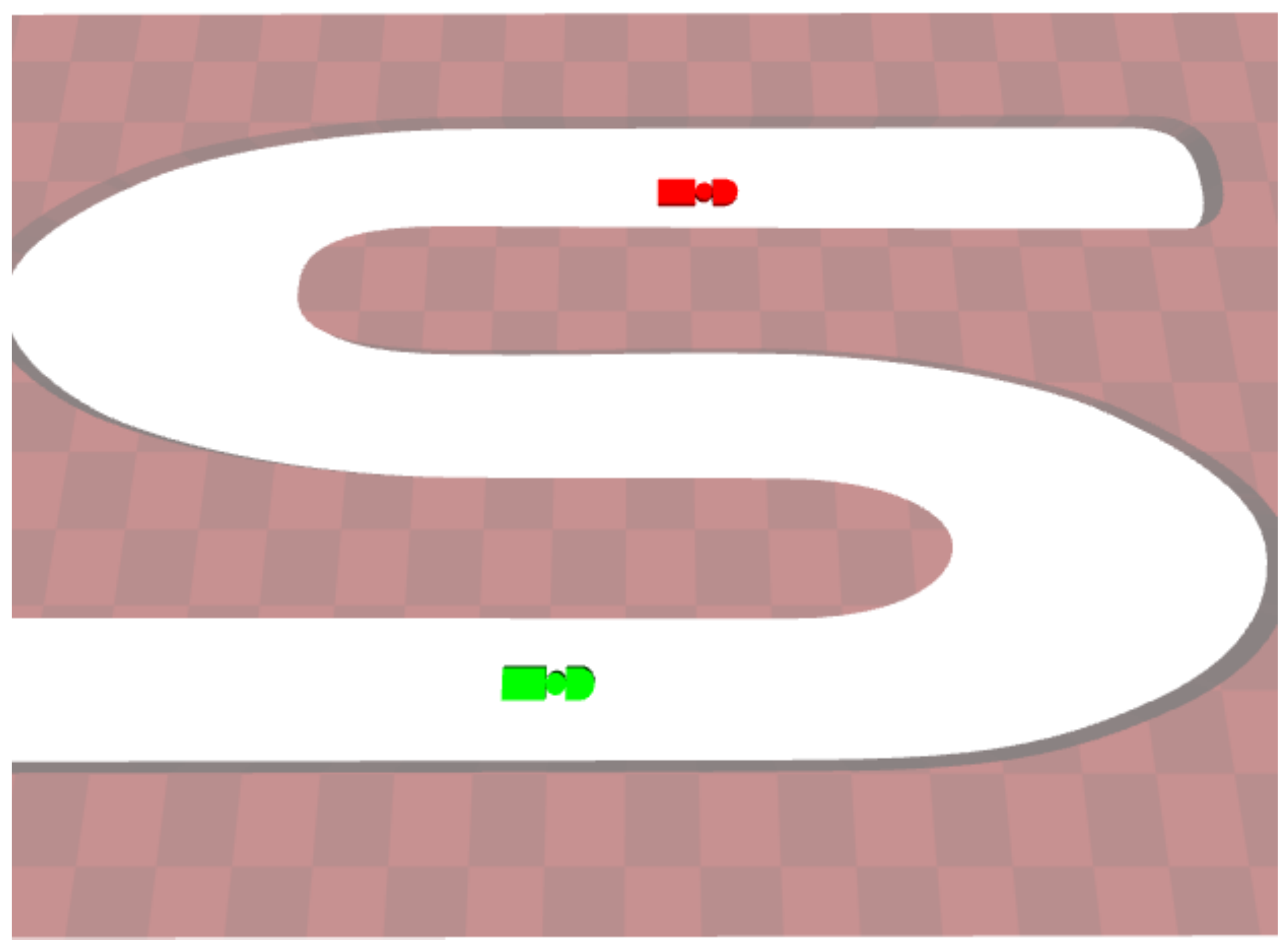}{$SE(2)\times R^1$}
    \subfigE{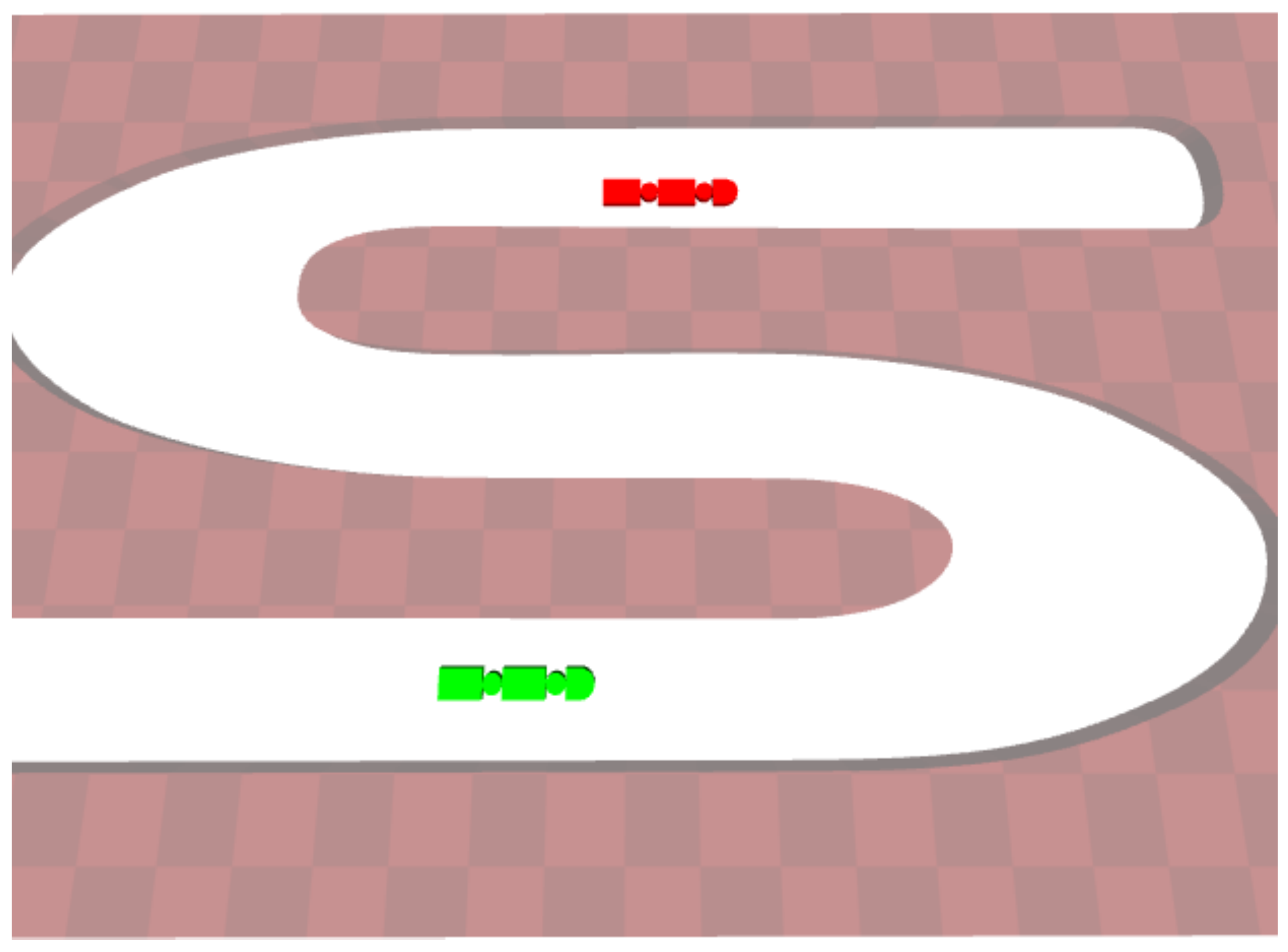}{$SE(2)\times R^2$}
    \subfigE{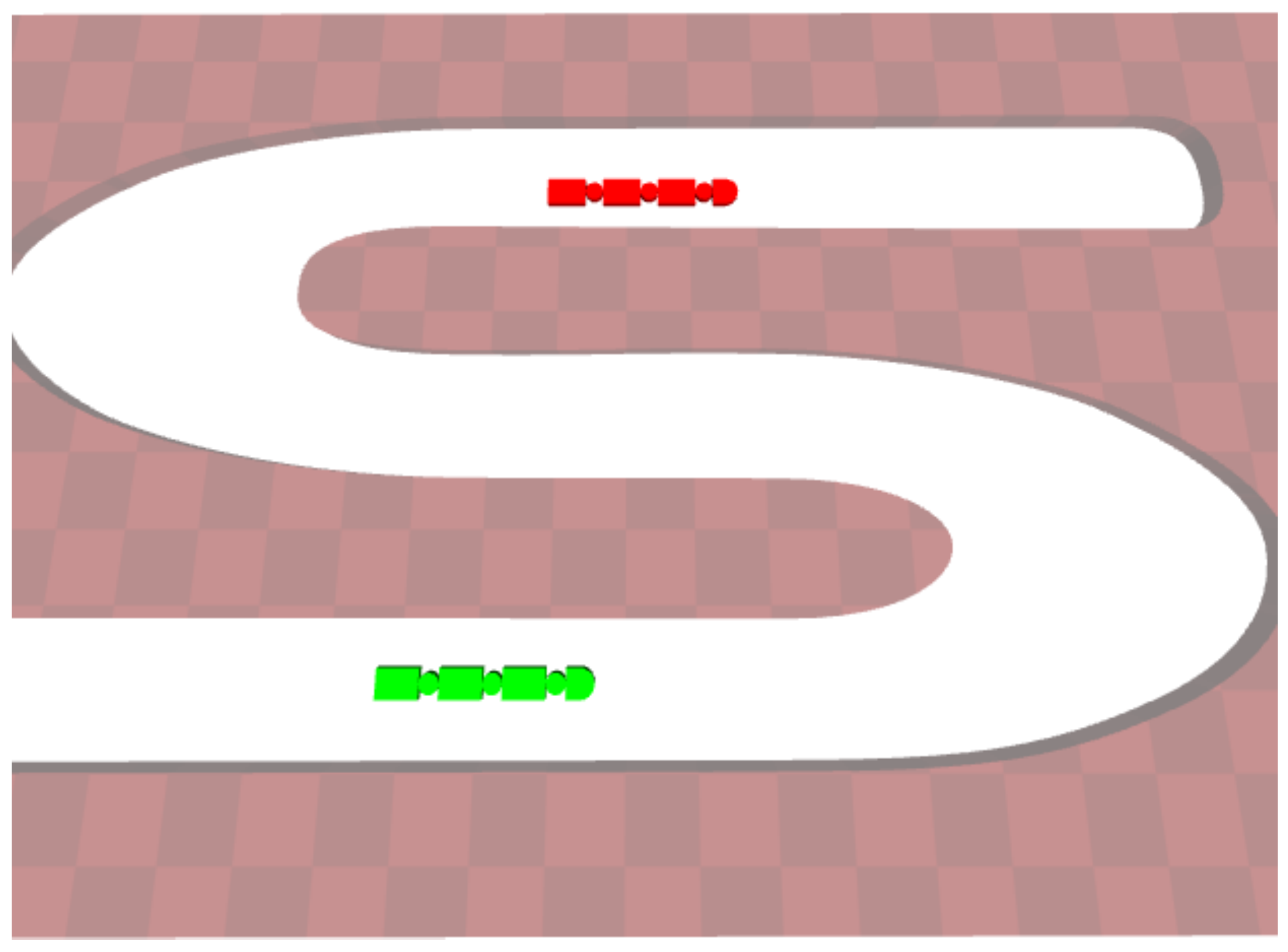}{$SE(2)\times R^3$}
    \subfigE{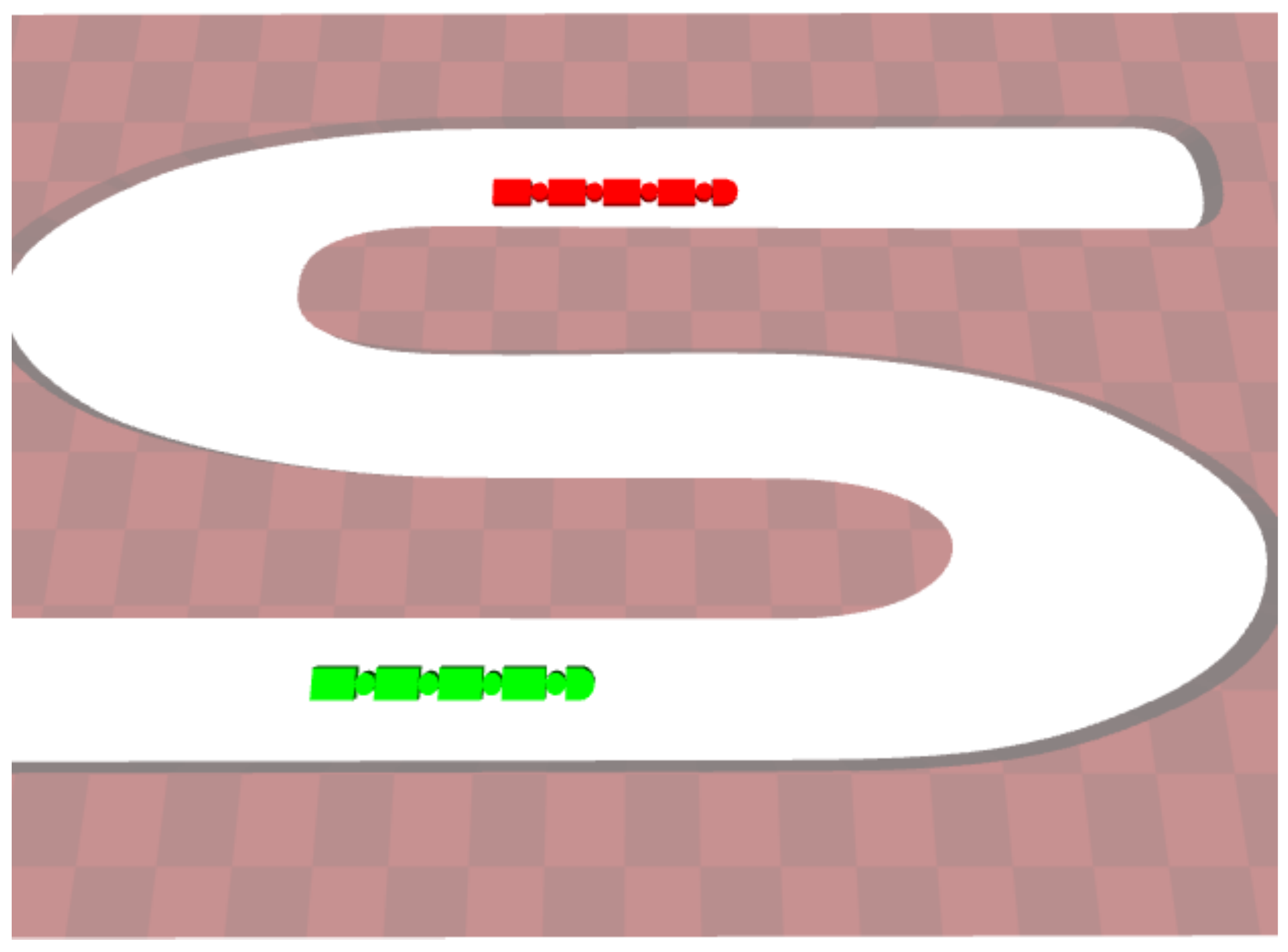}{$SE(2)\times R^4$} 
    
    \caption{The start (green) and goal (red) configuration of the 7-dof planar articulated body projected onto six quotient-spaces $\R^2$,$SE(2)$, and $SE(2)\times R^{\{1,\cdots,4\}}$.}
    \label{fig:results_planar_snake_subspaces}
\end{figure}

\begin{figure}[!ht]
    \centering
    \includegraphics[width=0.48\linewidth]{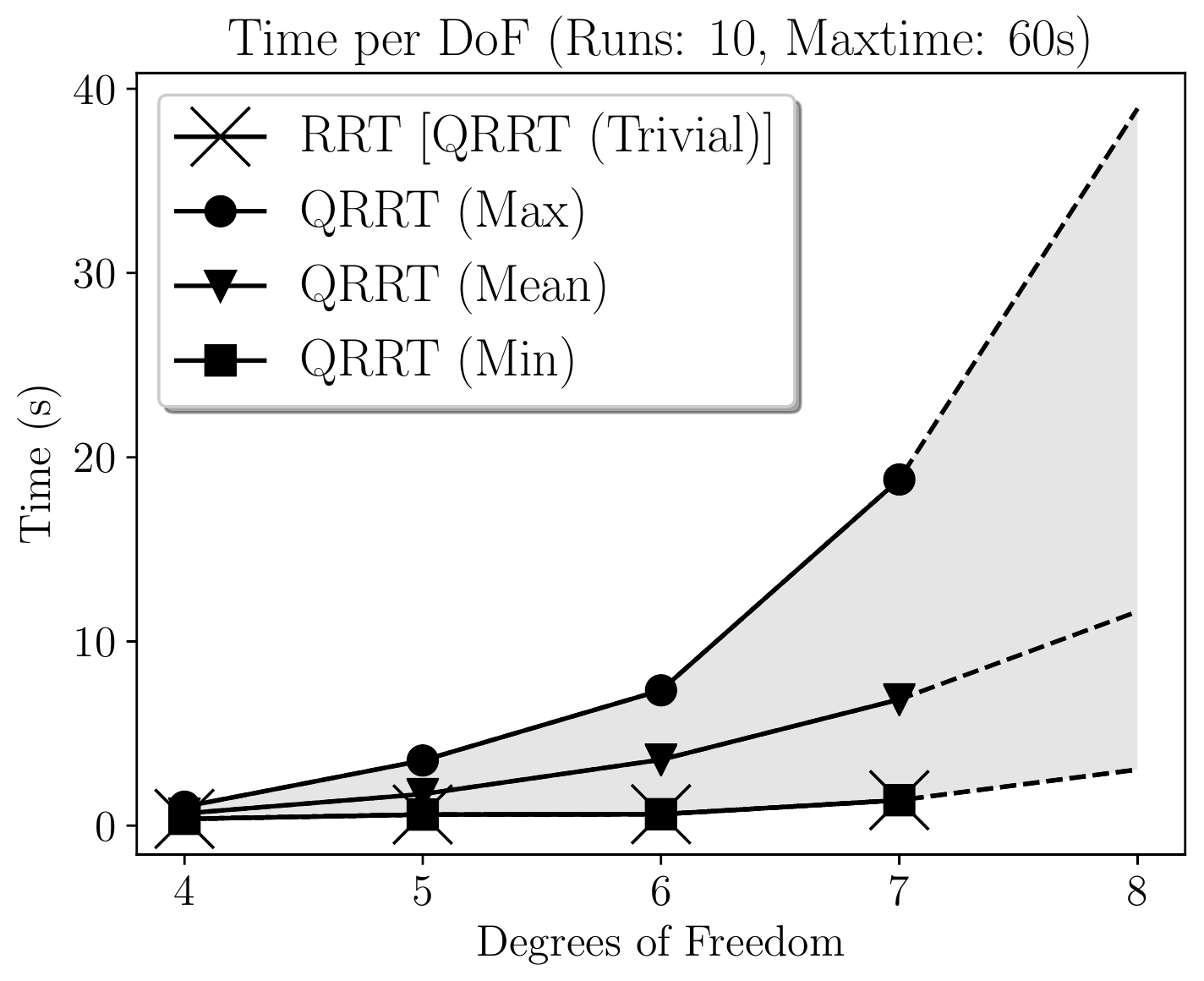} 
    \includegraphics[width=0.48\linewidth]{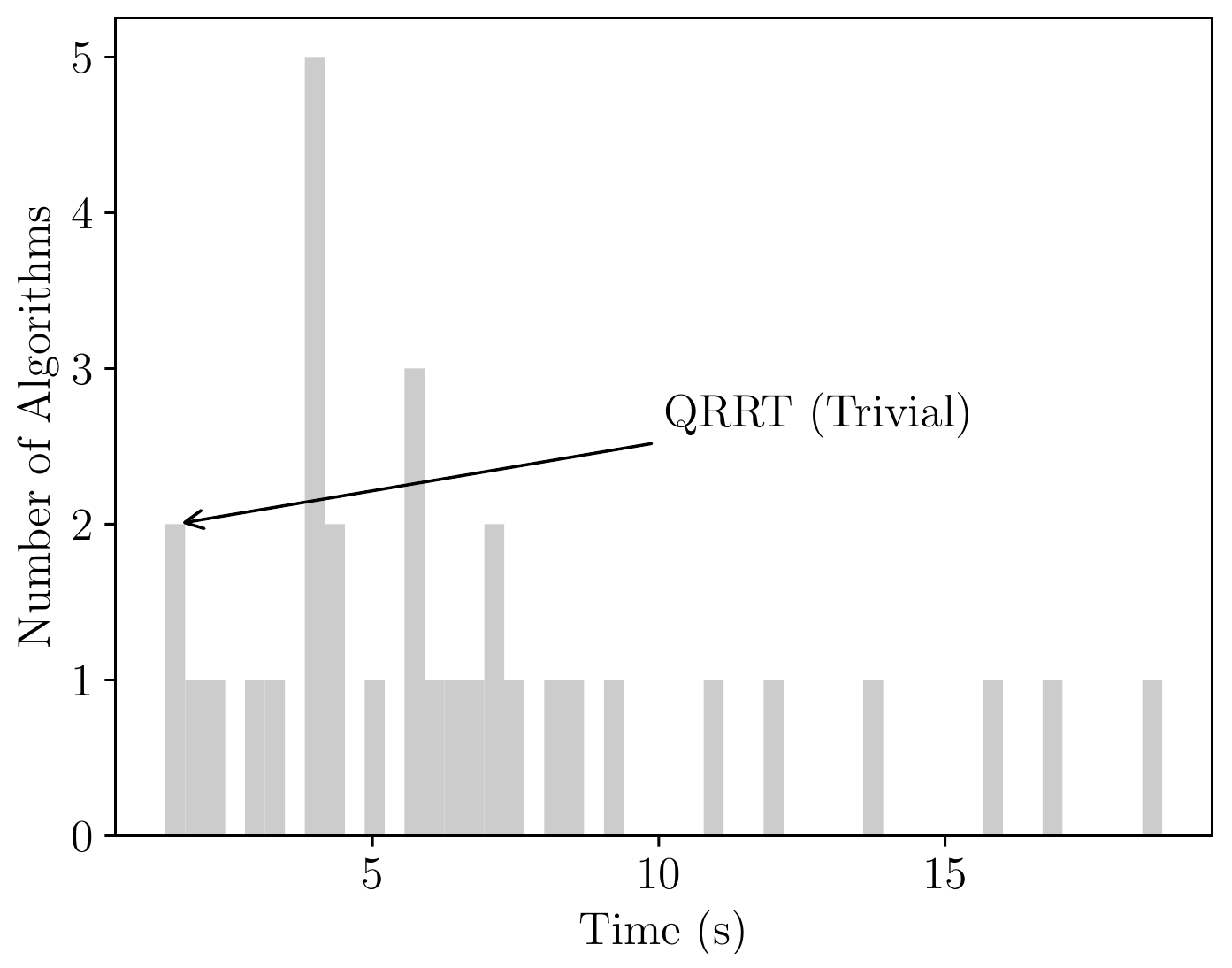}
    \caption{Runtime of QRRT using different sequential simplifications for the free-floating planar robot.}
    \label{fig:results_planar_snake}
\end{figure}

\subsubsection{Fixed-Base Spatial}

In our third experiment, we use a $7$-dof spatial manipulator with configuration space $\R^7$. We show all the efficient quotient-space mappings in Fig. \ref{fig:results_kuka_subspaces}. The results (with timelimit $300$s) are shown in Fig. \ref{fig:results_kuka} (left), where the runtime of all algorithms is around $180$ to $300$ seconds for $7$ dof (middle). The distribution of algorithms (right) shows that the trivial simplification performs near the mean.

\renewcommand\subfigE[2]{
\begin{subfigure}[b]{0.23\linewidth}
   \includegraphics[width=\linewidth]{#1}
   \caption{#2}
\end{subfigure}
}
\begin{figure}[!ht]
    \centering
    \subfigE{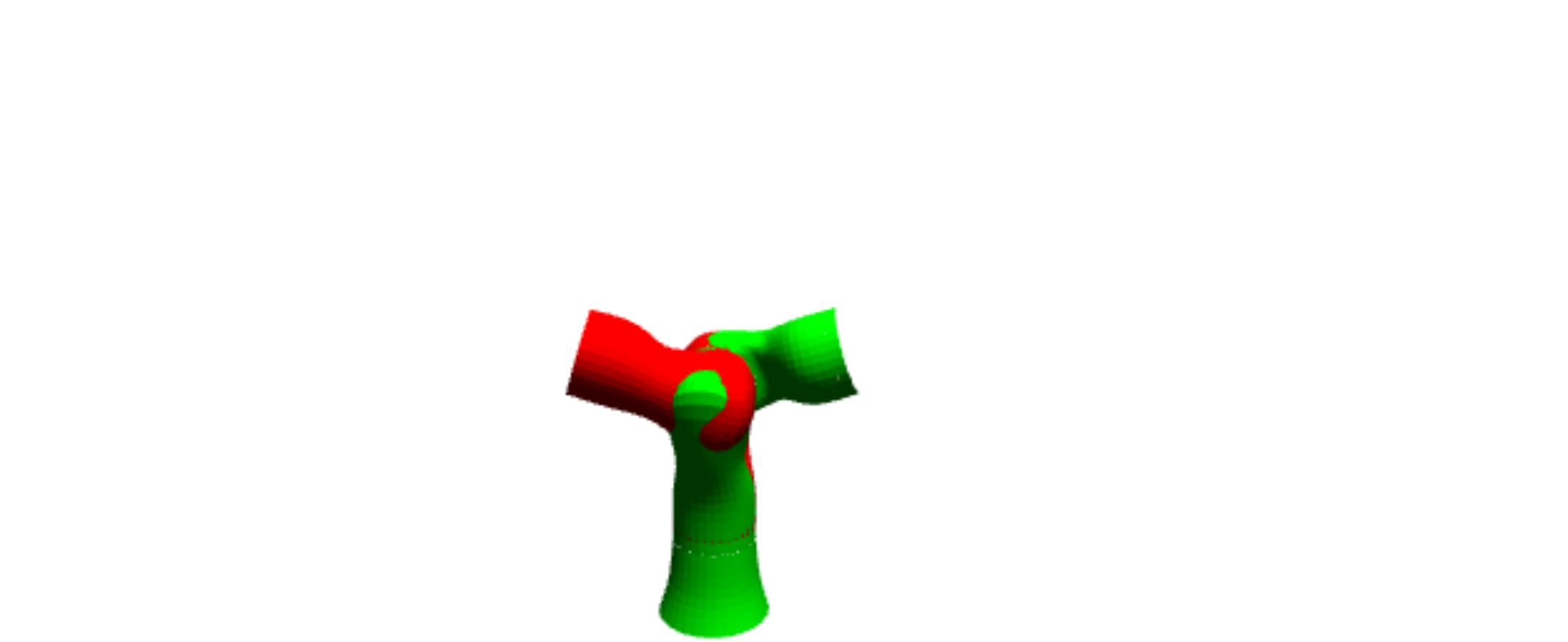}{$\R^1$}
    \subfigE{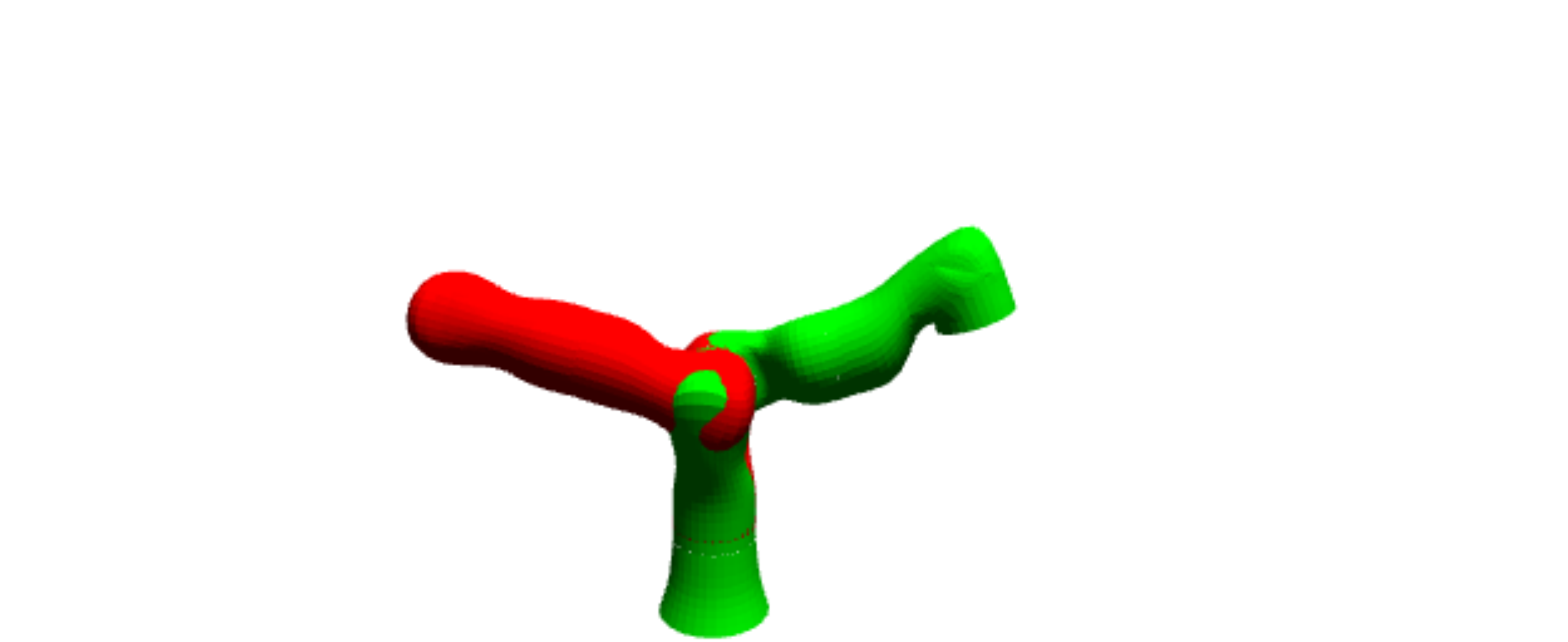}{$\R^2$}
    \subfigE{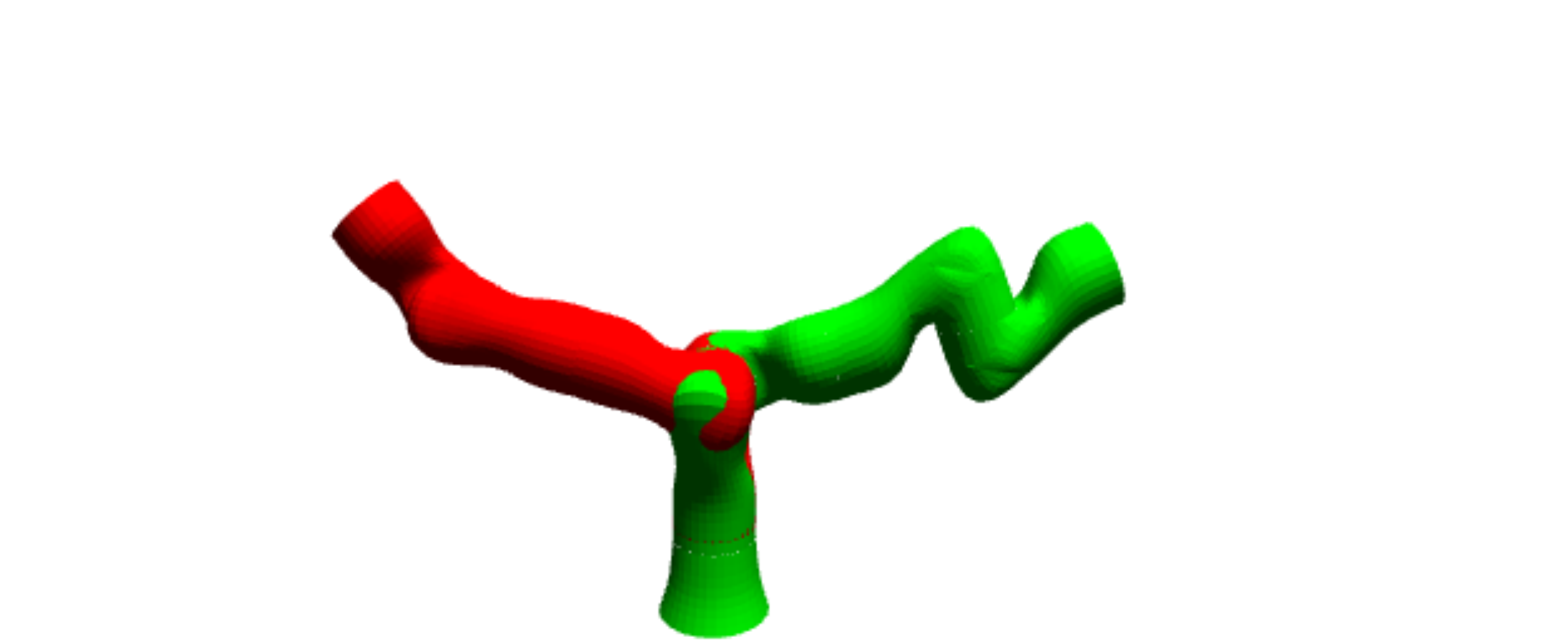}{$\R^3$}
    \subfigE{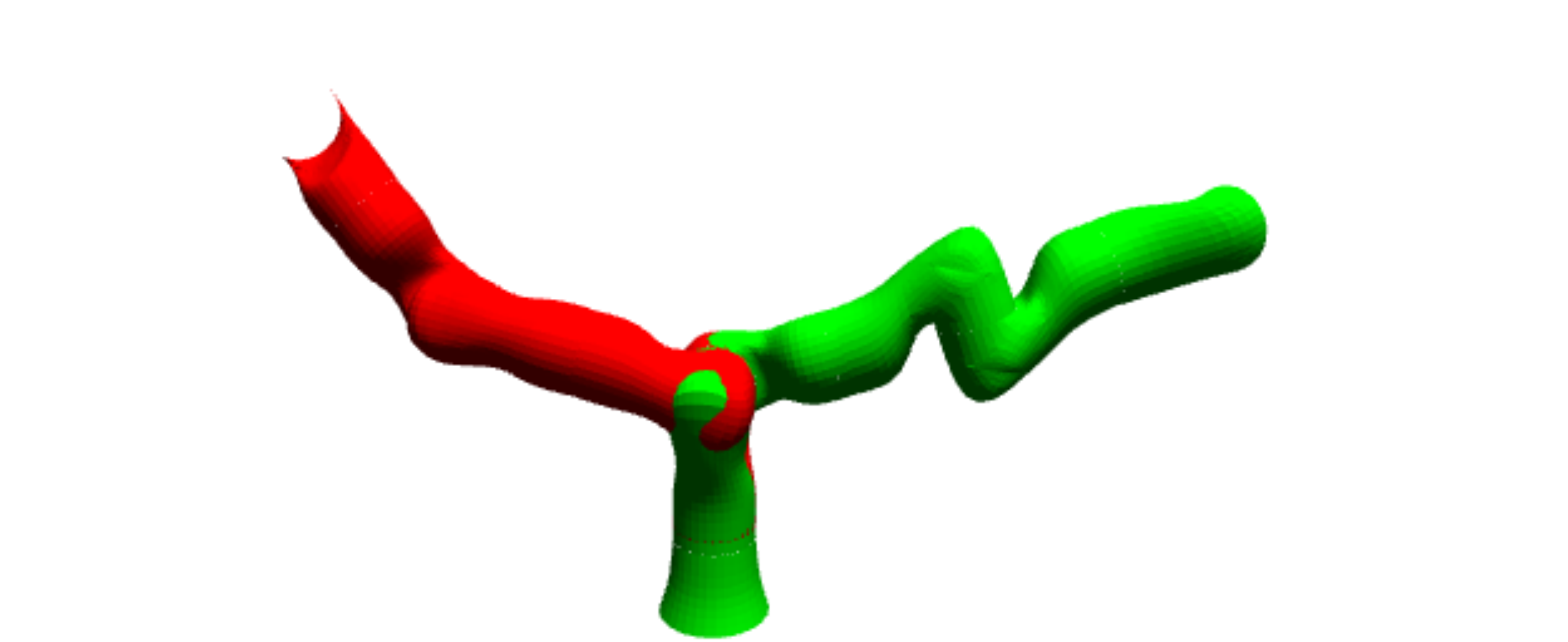}{$\R^4$}
    
    \subfigE{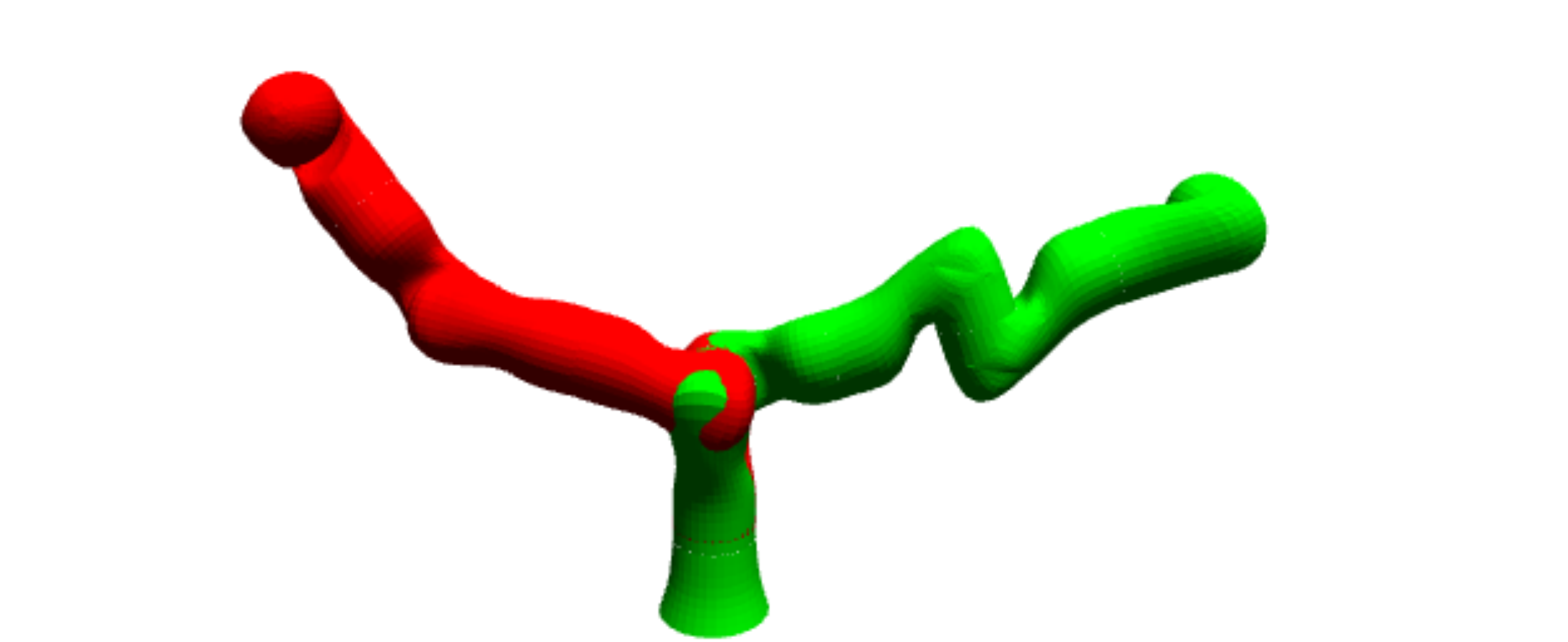}{$\R^5$}
    \subfigE{images/kuka_2dof.pdf}{$\R^6$}
    \subfigE{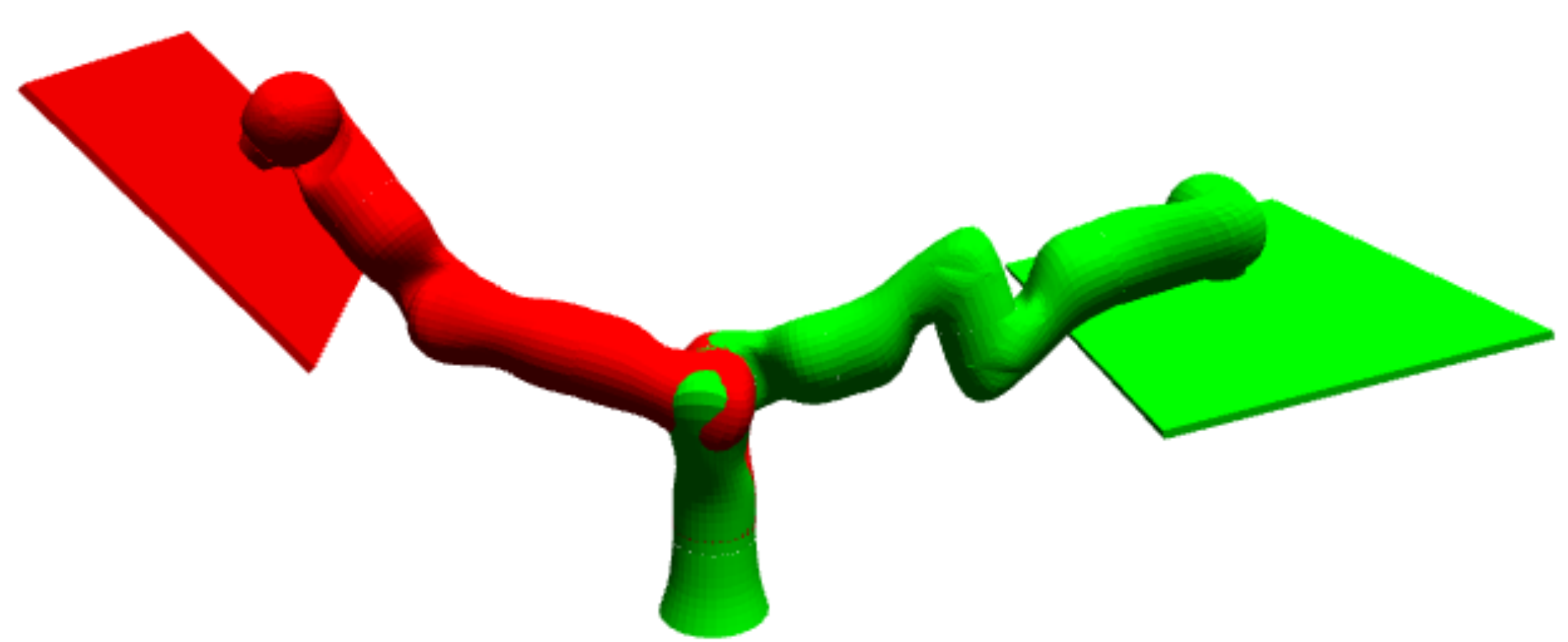}{$\R^7$}
    \subfigE{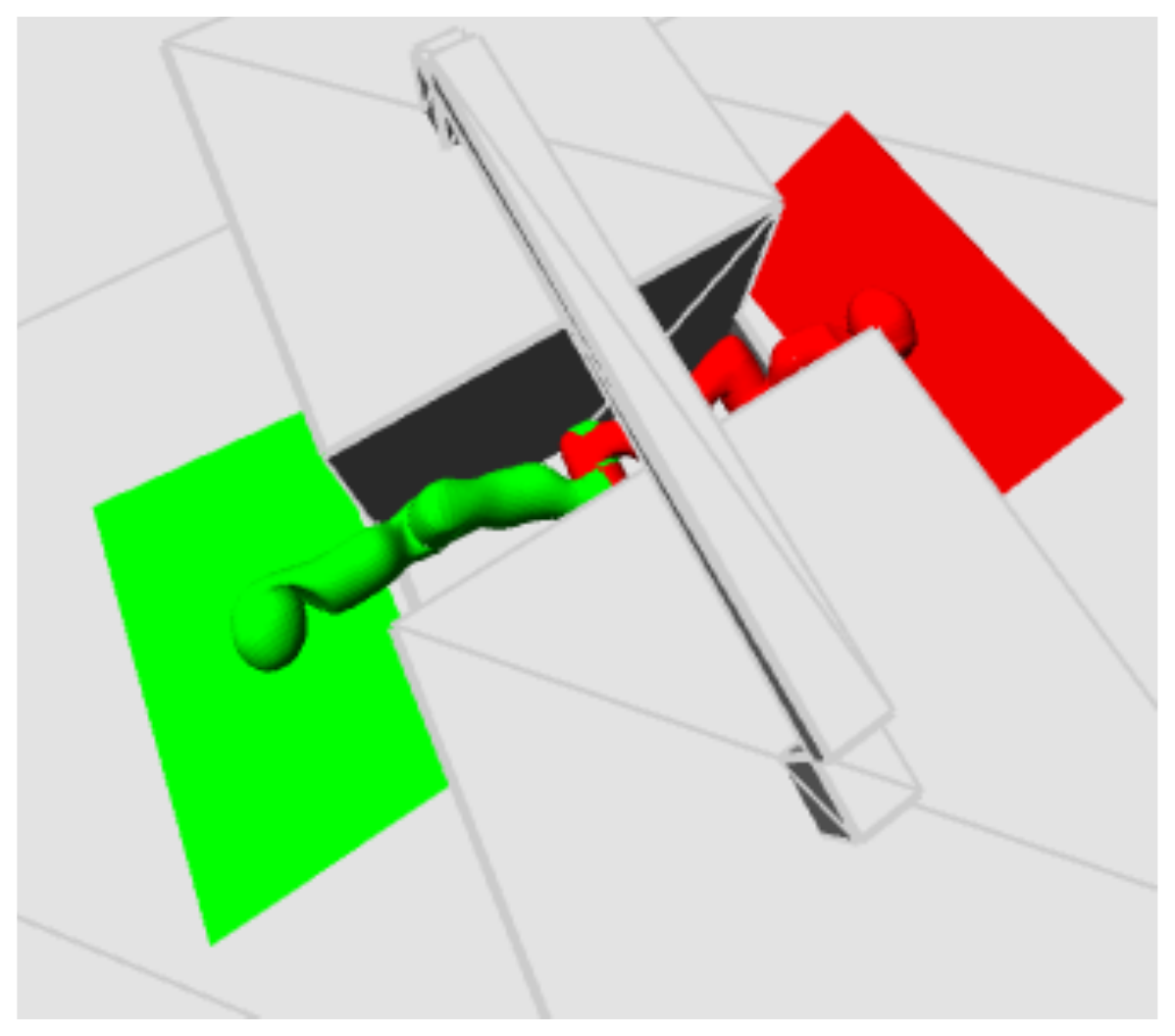}{Environment}    
    \caption{The start (green) and goal (red) configuration of the 7-dof KUKA manipulator arm projected onto five quotient-spaces $\R^{\{1,\cdots,5\}}$. The quotient-space $\R^6$ has been ignored, because the volume of the robot is equivalent to the robot on the quotient-space $\R^5$.}
    \label{fig:results_kuka_subspaces}
\end{figure}
\begin{figure}[!ht]
    \centering
    \includegraphics[width=0.48\linewidth]{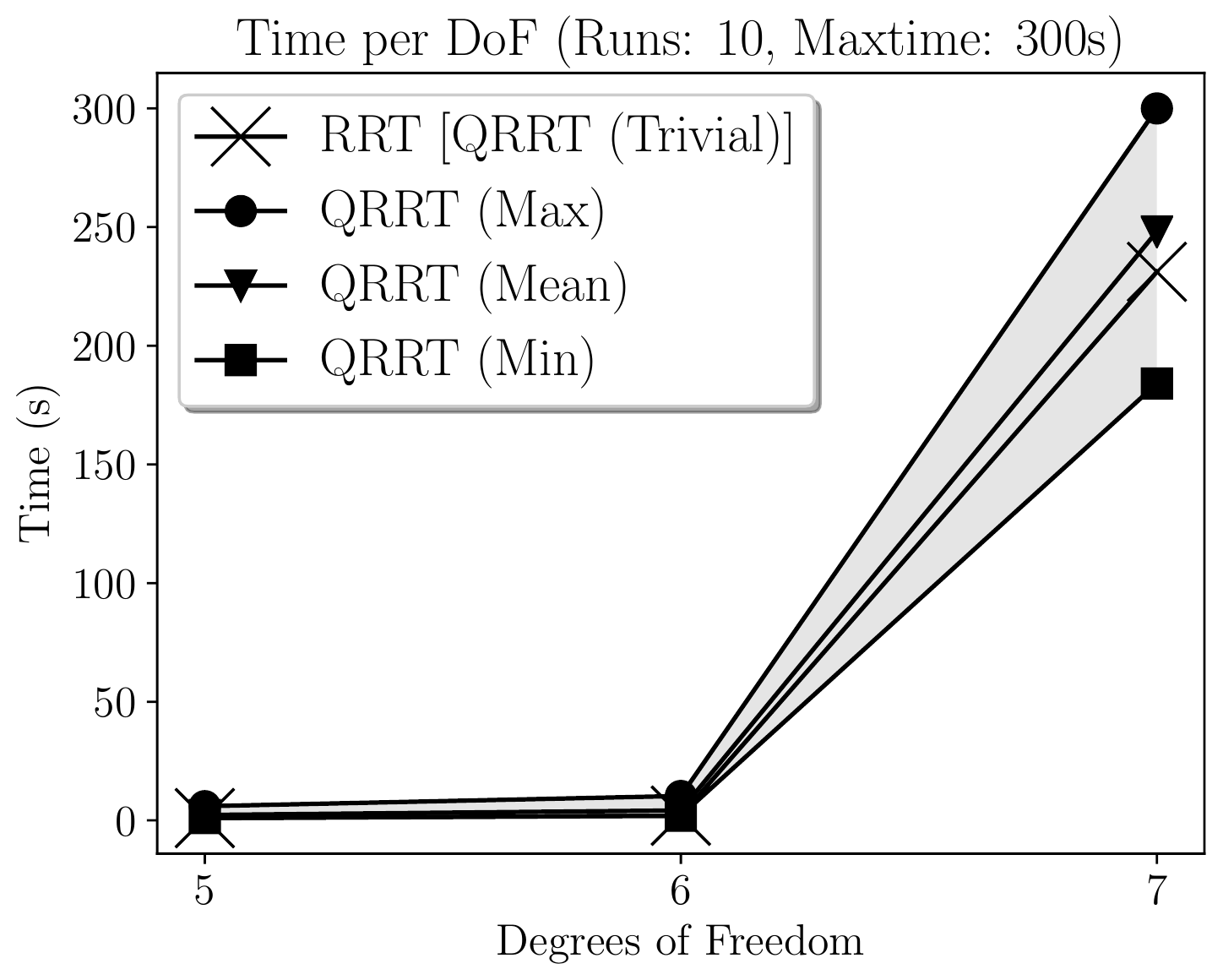}  
    \includegraphics[width=0.48\linewidth]{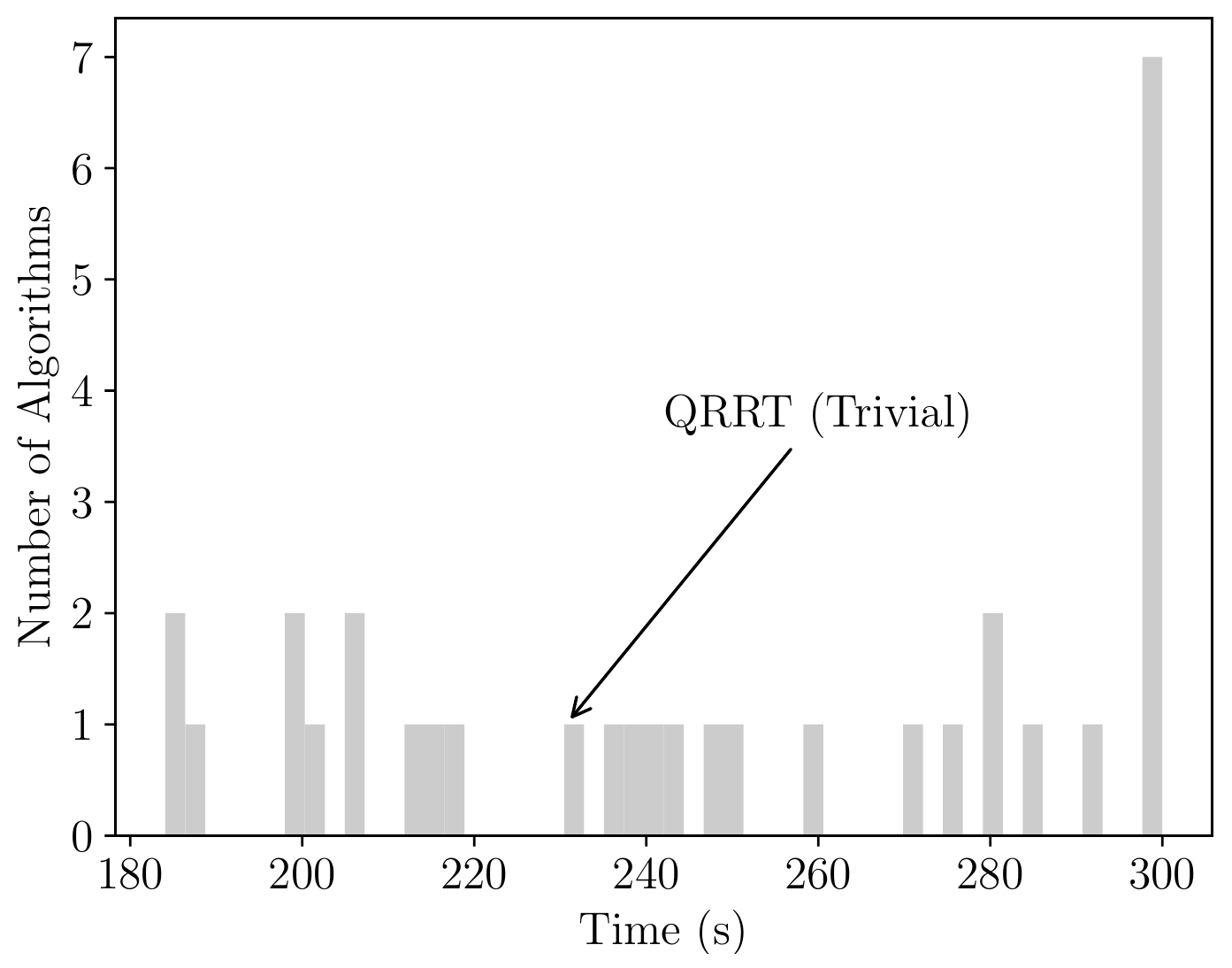}
    \caption{Runtime of QRRT using different sequential simplifications for the fixed-base spatial robot.}
    \label{fig:results_kuka}
\end{figure}

\subsubsection{Free-Floating Spatial}

Our last experiment is a $10$-dof spatial free-floating robot with configuration space $SE(3) \times \R^4$, which has to move through a twisted pipe. We define three efficient projections, visualized in Fig. \ref{fig:results_spatial_snake_subspaces}. The results in Fig. \ref{fig:results_spatial_snake} (left) show that the trivial simplification performs worst, with $300$s for $8$-dof (timelimit), while the best sequential simplification has runtime of $17$s, a reduction of at least one order of magnitude.

\renewcommand\subfigE[2]{
\begin{subfigure}[b]{0.23\linewidth}
   \includegraphics[width=\linewidth]{#1}
   \caption{#2}
\end{subfigure}
}
\begin{figure}[!ht]
    \centering
    \subfigE{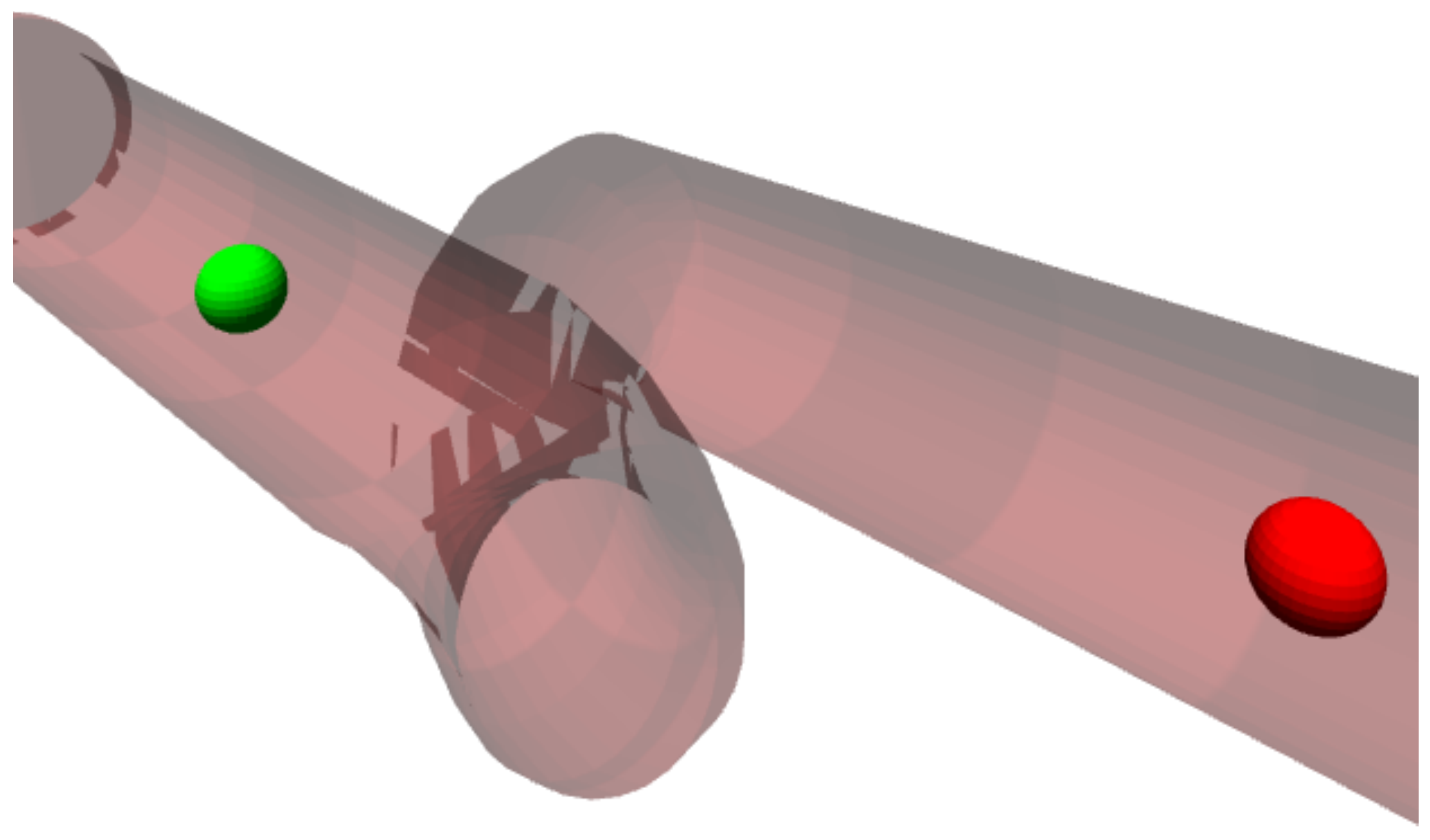}{$\R^3$}
    \subfigE{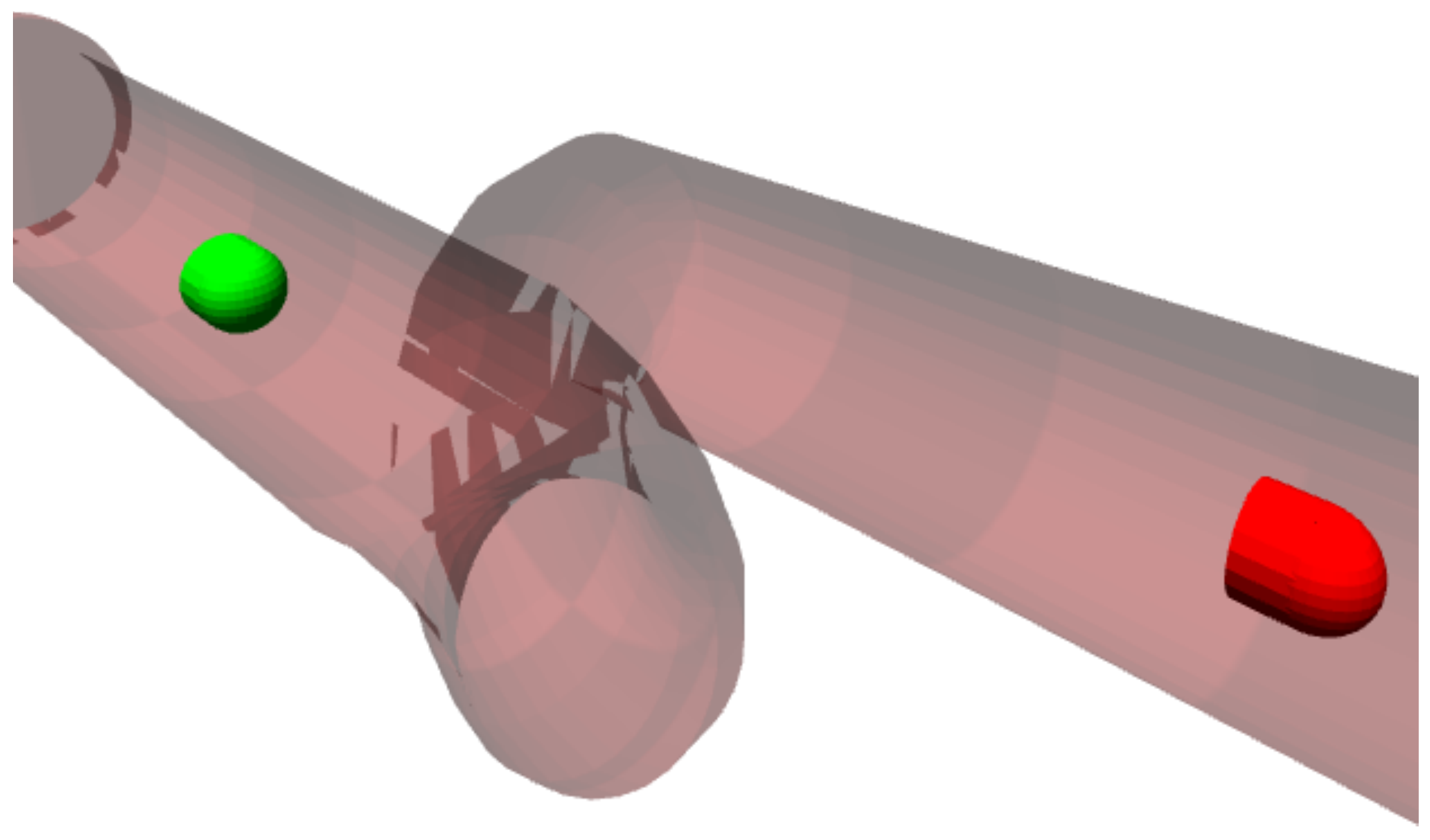}{$SE(3)$}
    \subfigE{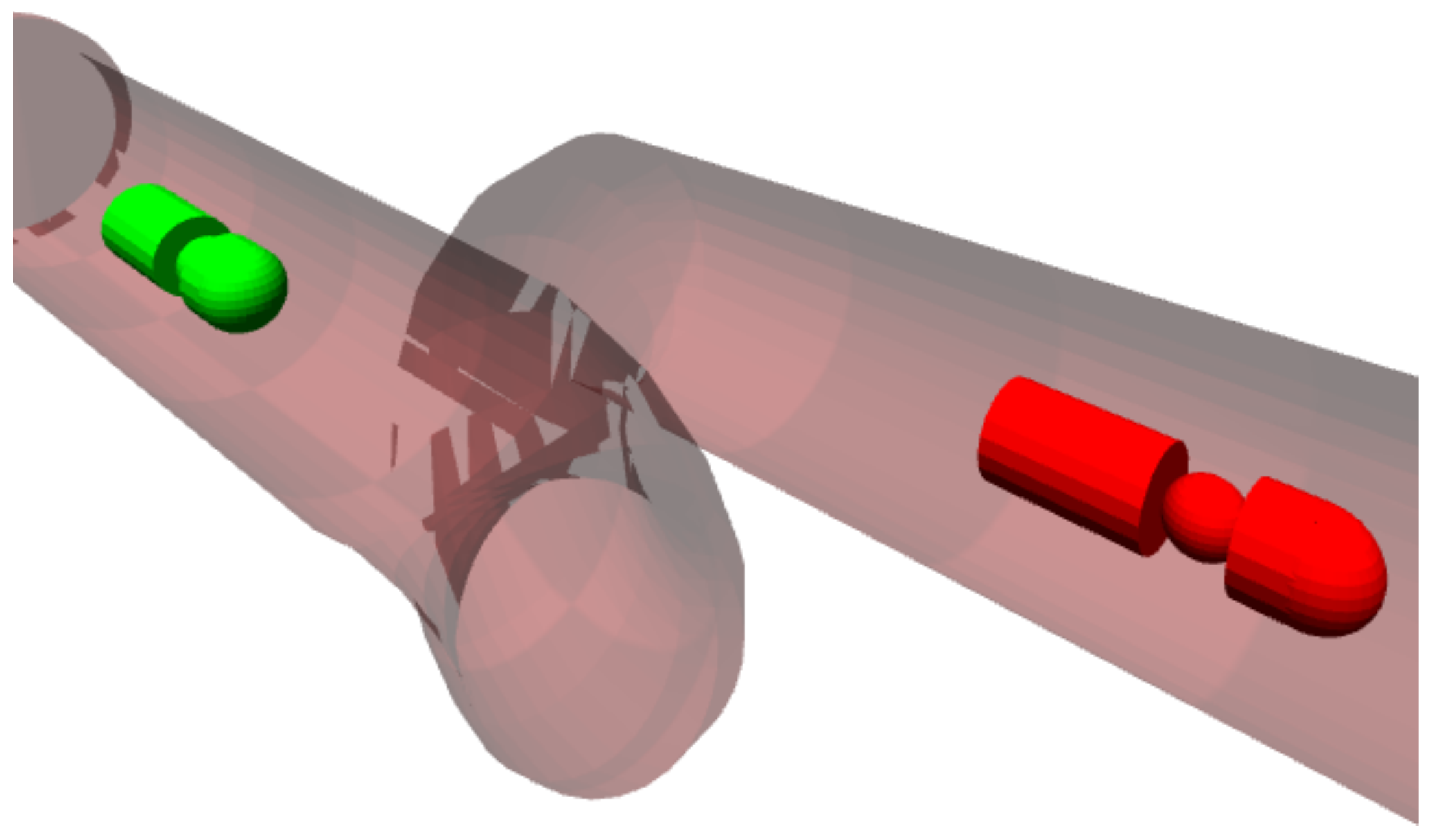}{$SE(3)\times \R^2$}
    \subfigE{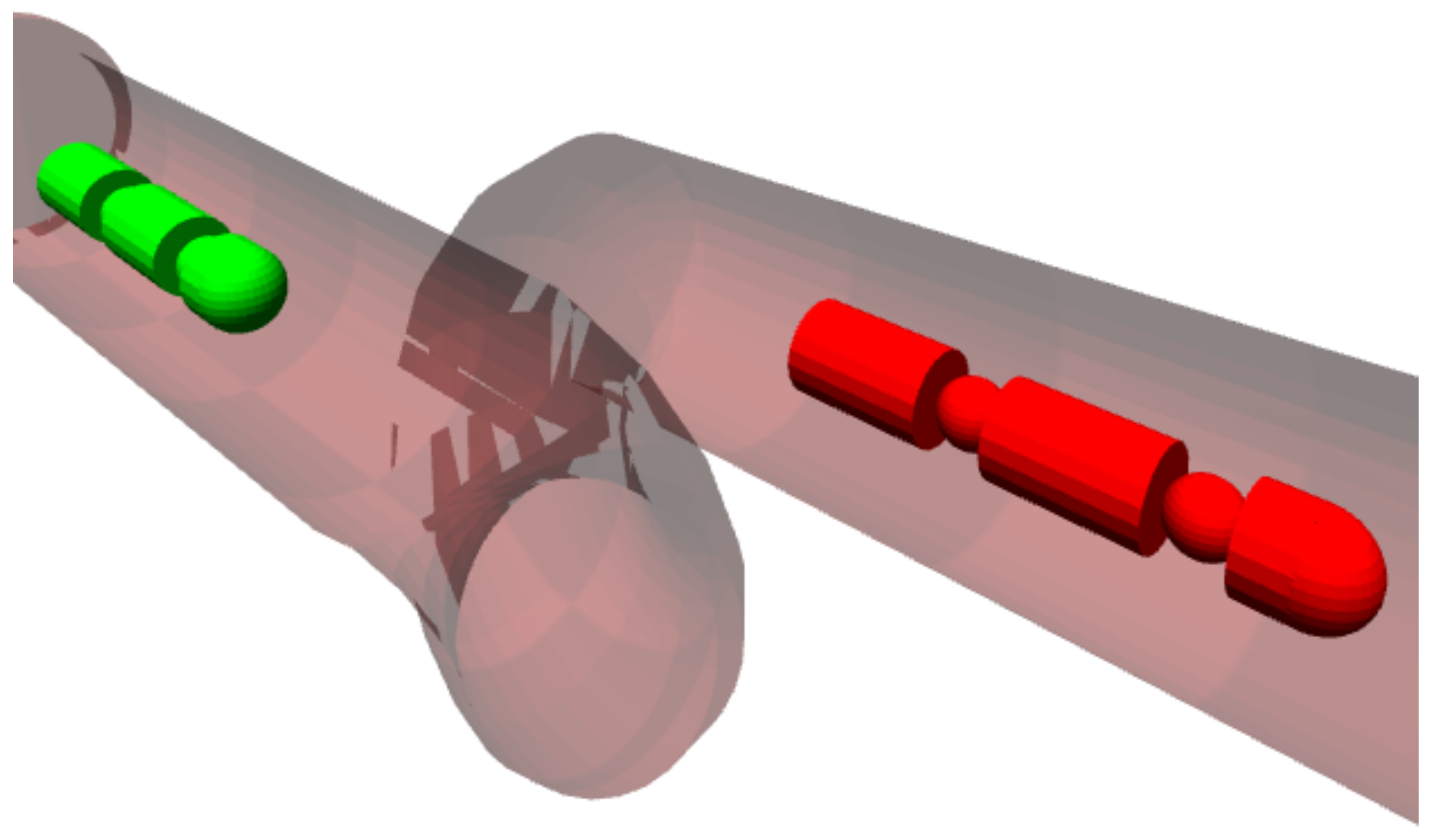}{$SE(3)\times \R^4$}
    \caption{The start (green) and goal (red) configuration of the 10-dof spatial articulated body projected onto four quotient-spaces $\R^3$,$SE(3)$, and $SE(3)\times \R^{\{2,4\}}$.}
    \label{fig:results_spatial_snake_subspaces}
\end{figure}
\begin{figure}[!ht]
    \centering
    \includegraphics[width=0.48\linewidth]{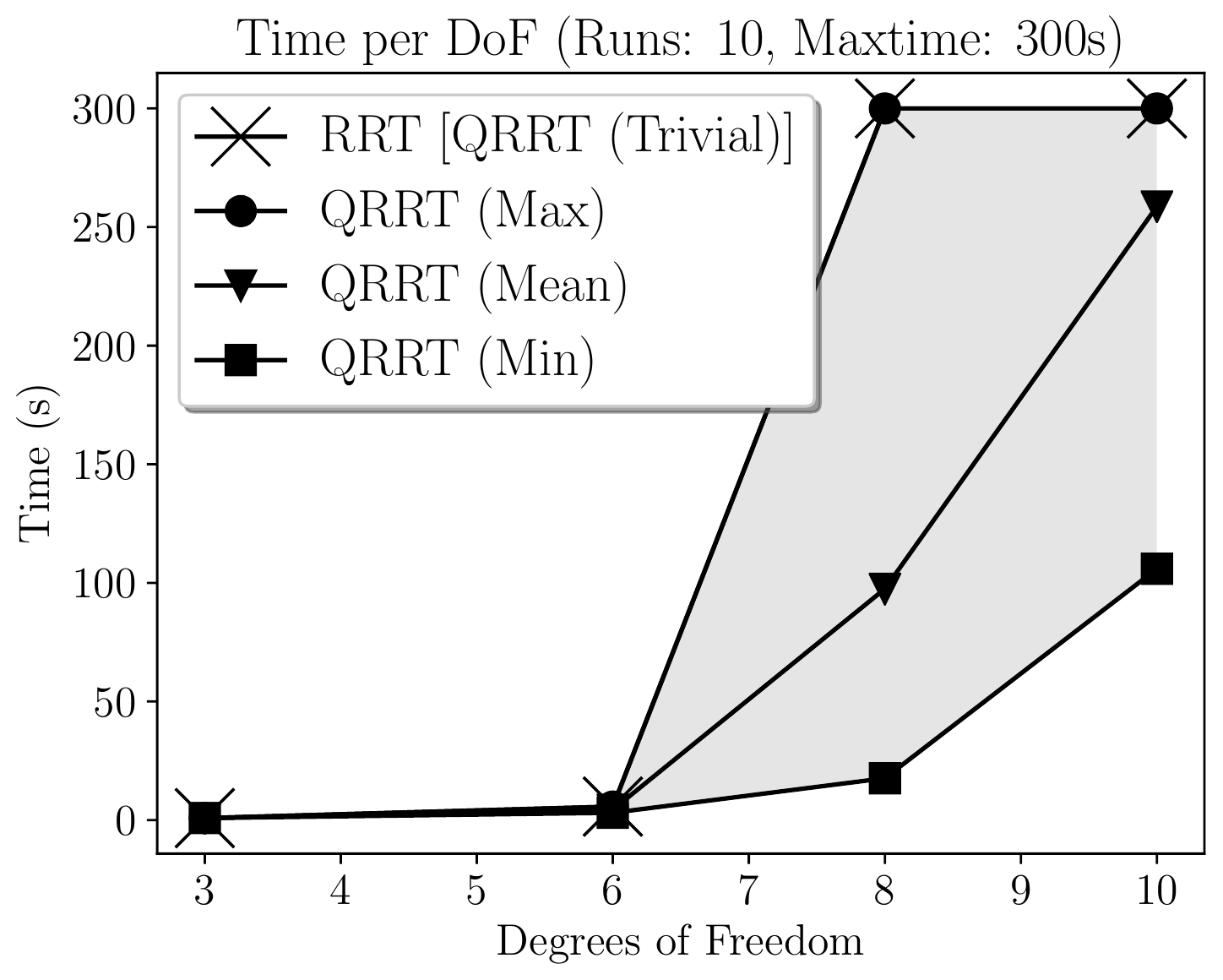}
    \includegraphics[width=0.48\linewidth]{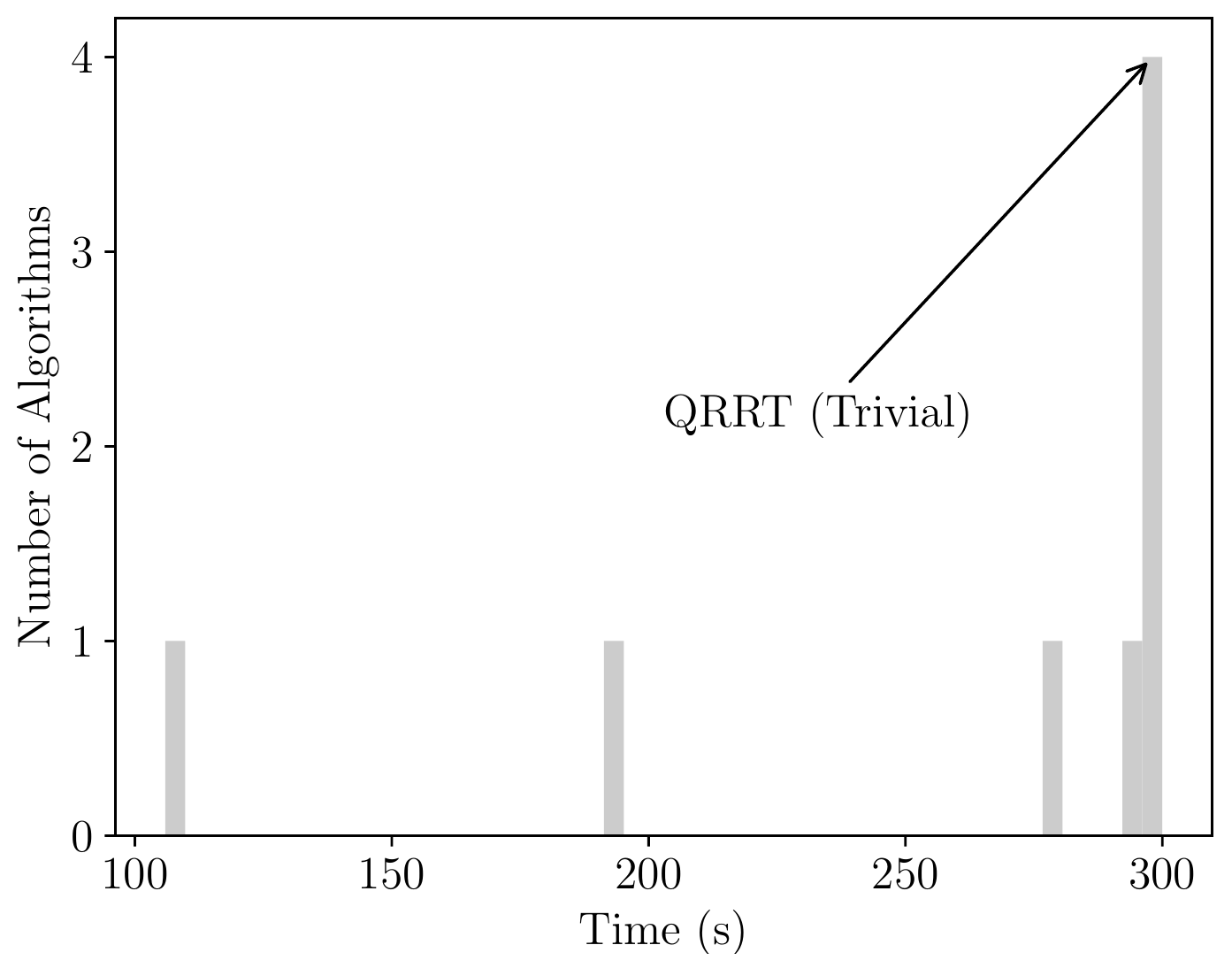}

    \caption{Runtime of QRRT using different sequential simplifications for the free-floating spatial robot.}
    \label{fig:results_spatial_snake}
\end{figure}

\subsection{QRRT in Narrow Passage}

The experiments show that the trivial simplification performed worst for the $8$-dof planar manipulator and the $10$-dof spatial free-floating robot, but performed best for the $7$-dof planar articulated body. We would like to understand why that is the case, and when we should prefer the trivial simplification over a non-trivial one.

We observe that any additional projection incurs an additional cost. On the one hand, if an environment contains no narrow passages, it can usually be solved quickly, while building simplifications incurs overhead costs. On the other hand, if there are narrow passages in the environment, the cost of constructing simplifications is negligible, and runtime can be reduced.

To test this idea, we have build a simple environment with a narrow passage of width $2\alpha > 0$, as shown in Fig. \ref{fig:narrow_passage}. Prior work by \cite{saha_2005} showed that the runtime of a trivial simplification (i.e. RRT) increases exponentially with decreasing $\alpha$. 

We have used a 4-dof robot in the plane, with configuration space $X = SE(2) \times \R^2$, where we define four quotient-space sequences: the trivial one $\{X\}$, $\{\R^2,X\}$, $\{SE(2),X\}$, and $\{\R^2,SE(2),X\}$. The corresponding algorithm for each simplification will be named $\QRRT{4}$,$\QRRT{24}$,$\QRRT{34}$, and $\QRRT{234}$, respectively, where the number indicates the dimensionality of each space.
The radius of the robot is $0.1$m, such that the passage could be traversed whenever $\alpha \geq 0.1$m. The results for different values of $\alpha$ are shown in Fig. \ref{fig:benchmark_narrow_passage}. On the left, values of $\alpha$ in $[0.34,0.25]$ show the trivial simplification (bold black curve) performs best (lower runtime is better). However, for smaller values of $\alpha$ in $[0.125,0.115]$, we observe (on the right) how the trivial simplification performs worst, while the cost of building quotient-spaces pays off: The smaller the narrow passage, the better the non-trivial simplifications perform. For $\alpha=0.115$, the runtime reduction from trivial ($\QRRT{4}$) to best simplification ($\QRRT{24}$) is $120$s to $30$s, or around one order of magnitude.

\renewcommand\subfigE[2]{
\begin{subfigure}[b]{0.3\linewidth}
   \includegraphics[width=\linewidth]{#1}
   \caption{#2}
\end{subfigure}
}
\begin{figure}[!ht]
    \centering
    \subfigE{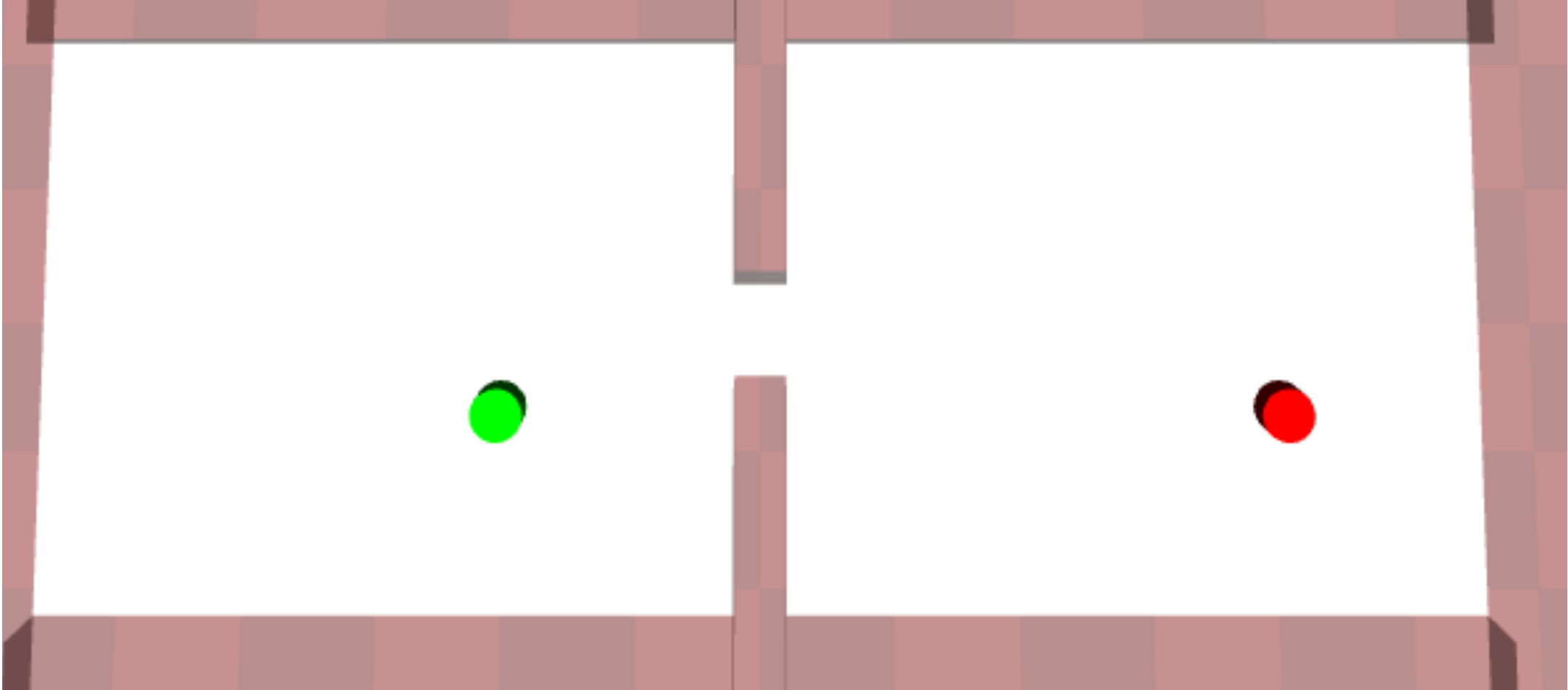}{}
    \subfigE{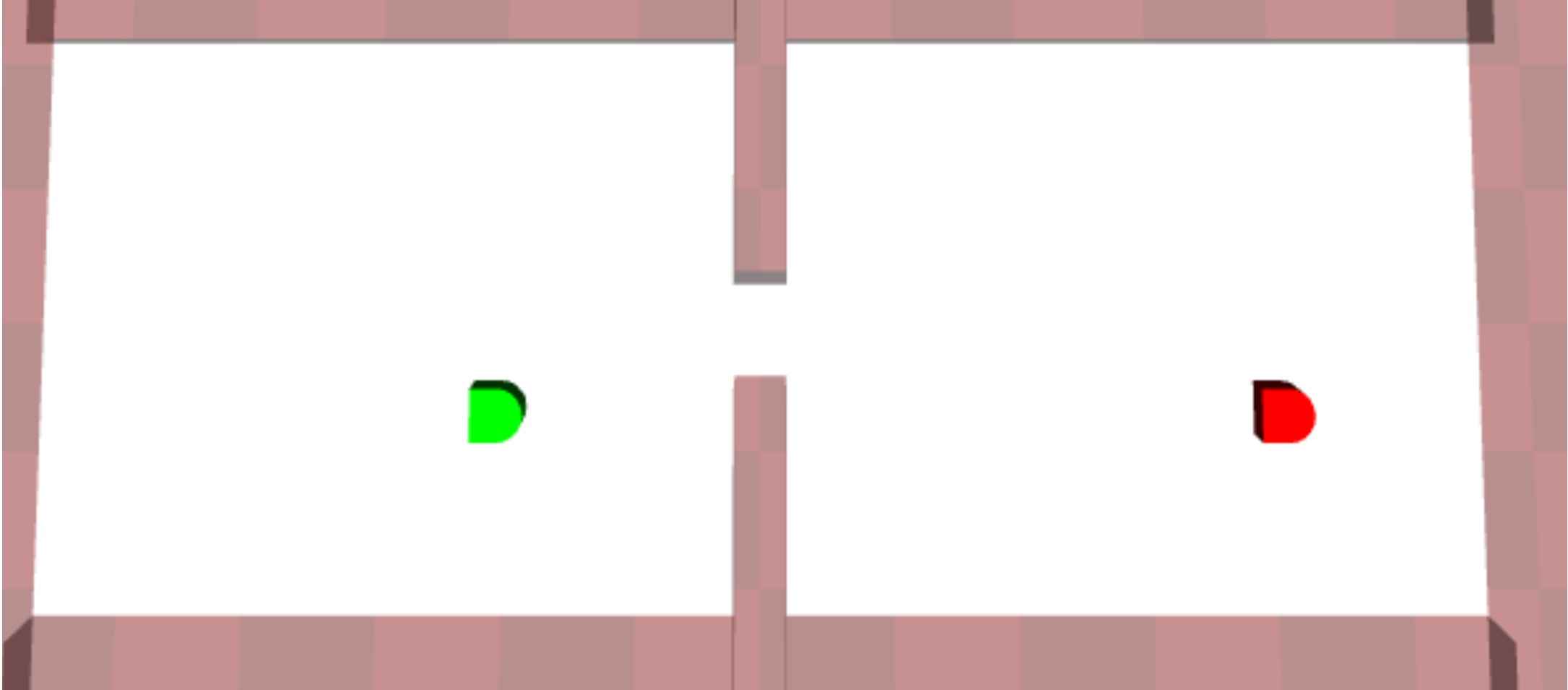}{}
    \subfigE{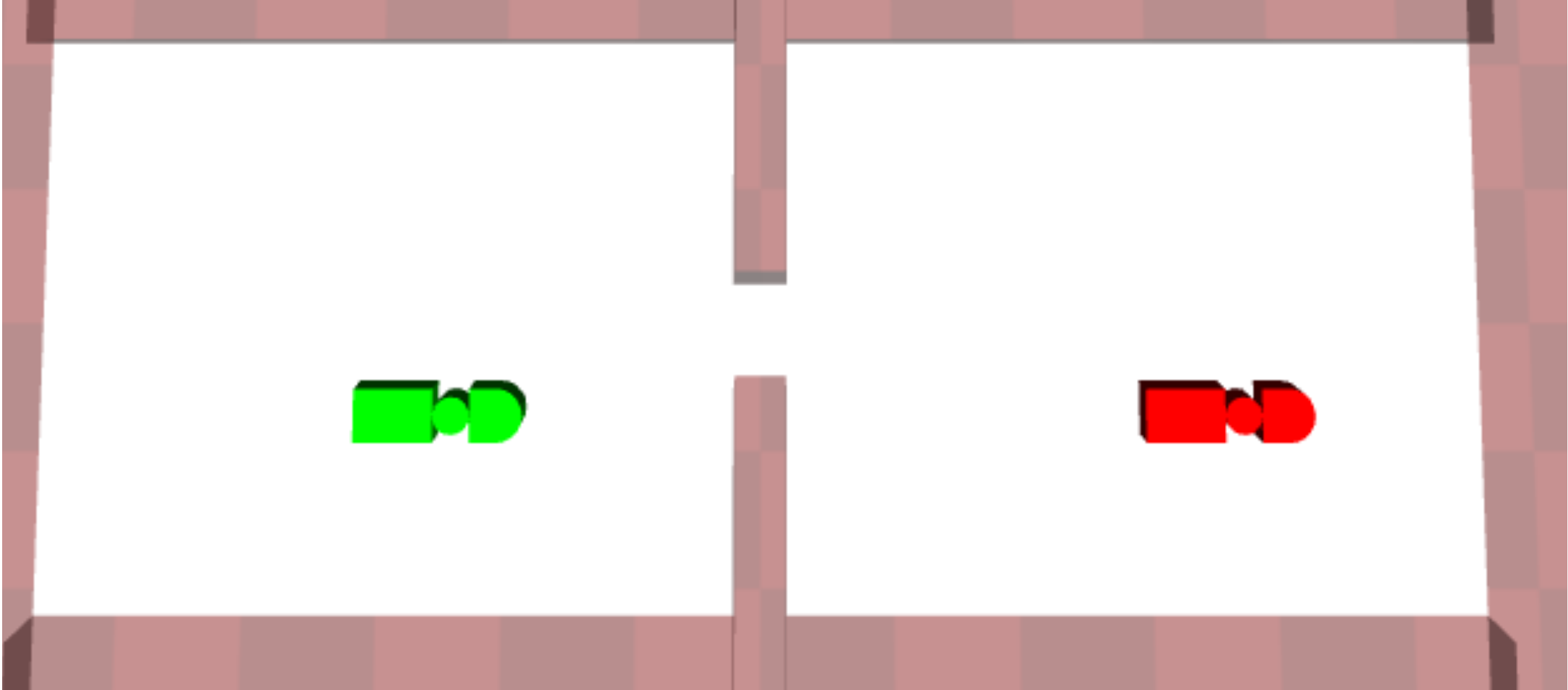}{}
    \caption{4-dof planar robot in narrow passage. The configuration space of the robot has three quotient-spaces, $\R^2$ (left), $SE(2)$ (middle) and $SE(2)\times \R^1$ (right). The size of the opening is $2\alpha$, and the radius of the disk is $r=0.1$. \label{fig:narrow_passage}}
    
\end{figure}
\renewcommand\subfigE[2]{
\begin{subfigure}[b]{0.47\linewidth}
   \includegraphics[width=\linewidth]{#1}
   \caption{#2}
\end{subfigure}
}
\begin{figure}[!ht]
    \centering
    \subfigE{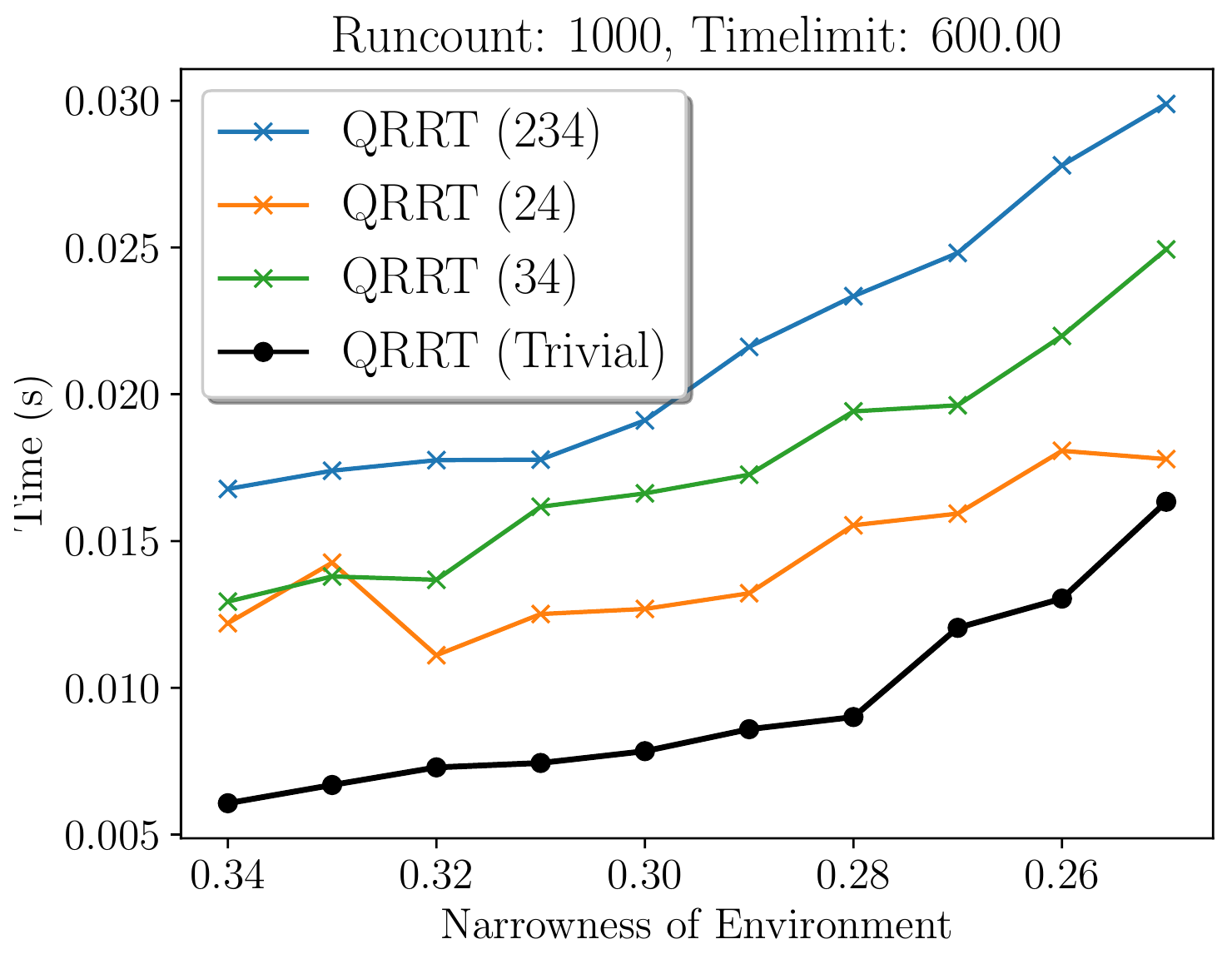}{}
    \subfigE{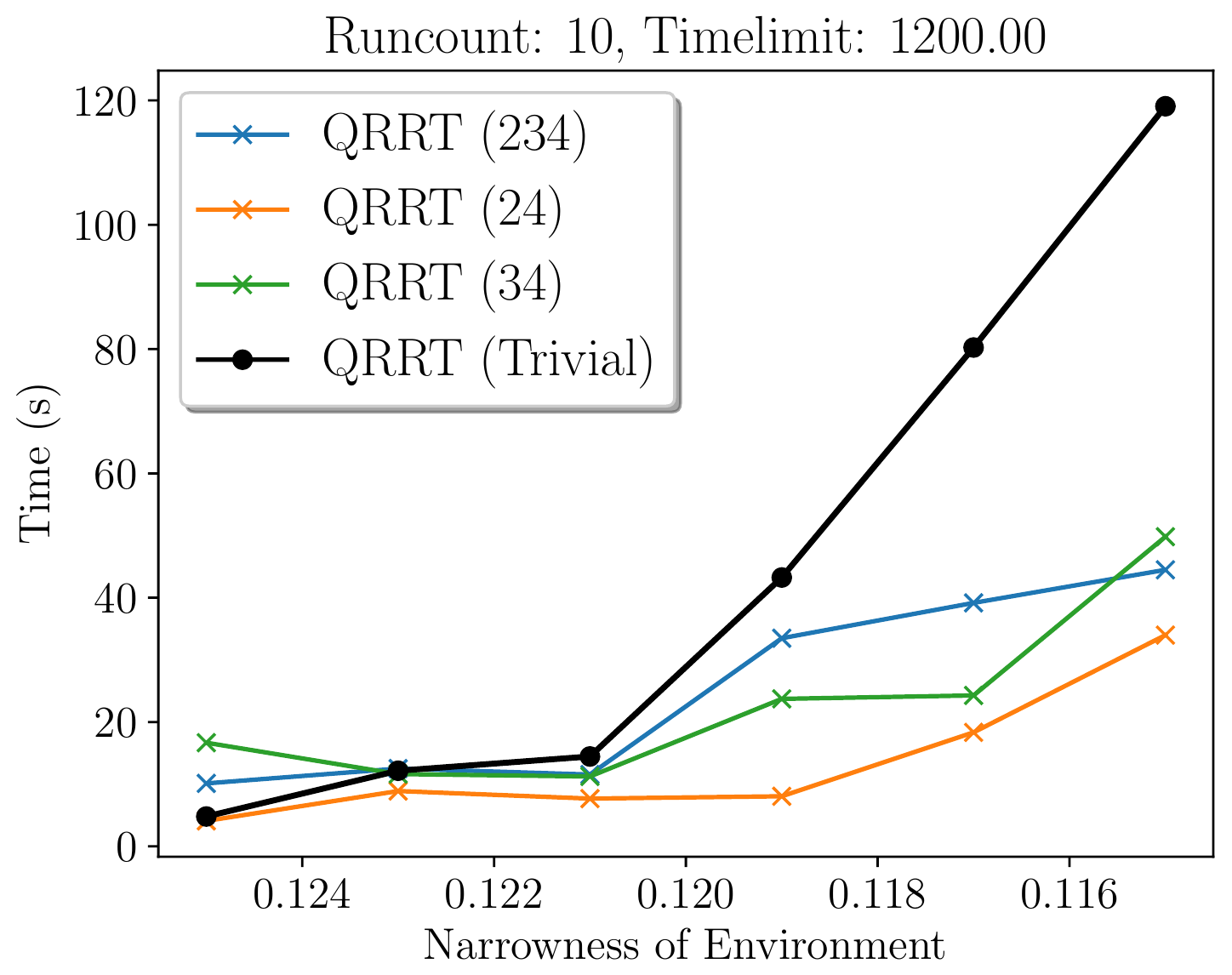}{}
    
    \caption{Benchmark of QRRT on a narrow passage environment. Left: QRRT using the trivial simplification has lowest runtime in a non-narrow passage ($\alpha \in [0.34,0.25]$). Right: QRRT using non-trivial simplifications have lowest runtime in narrow passages ($\alpha \in [0.125,0.115]$).
    \label{fig:benchmark_narrow_passage}
    }
    
\end{figure}

\section{Conclusion}

We have developed the Quotient-space Rapidly-exploring Random Trees (QRRT) algorithm. QRRT generalizes the RRT algorithm \cite{lavalle_1998} to sequential simplifications. 

QRRT with different sequential simplifications will yield different runtimes. While we showed that the runtime depends on narrow passages in the environment, there might be other factors influencing it like the robots geometry. Therefore, we cannot say which simplification minimizes the runtime for QRRT. To find such a minimal runtime simplification, we might need to search through sequential simplifications automatically or even investigate non-admissible simplifications. 

While choosing a good simplification is crucial, we were able to show that some simplifications can reduce the runtime of QRRT by at least one order of magnitude. This shows that finding good simplifications is an important factor in building fast planning algorithms.

\bibliographystyle{IEEEtranS}
{\small
\bibliography{IEEEabrv,bib/planning,bib/quotientspace}
}
\end{document}